\documentclass{article}

\usepackage{microtype}
\usepackage{graphicx}
\usepackage{subcaption}
\usepackage{booktabs} 

\usepackage{hyperref}

 \usepackage[preprint]{icml2026}

\usepackage{amsmath}
\usepackage{amssymb}
\usepackage{mathtools}
\usepackage{amsthm}
\usepackage{bbm}

\usepackage{multirow}
\usepackage[capitalize,noabbrev]{cleveref}

\theoremstyle{plain}
\newtheorem{theorem}{Theorem}[section]
\newtheorem{proposition}[theorem]{Proposition}

\theoremstyle{definition}

\theoremstyle{remark}
\newtheorem{remark}[theorem]{Remark}

\usepackage[textsize=tiny]{todonotes}

\icmltitlerunning{Semi-Supervised Gaze Estimation via Disentangled Subspace Contrastive Learning}

\begin{document}

\twocolumn[
  \icmltitle{Semi-Supervised Gaze Estimation via Disentangled Subspace Contrastive Learning}



  \icmlsetsymbol{equal}{*}

 \begin{icmlauthorlist}
    \icmlauthor{Qida Tan}{1}
    \icmlauthor{Hongyu Yang}{1,2}
    \icmlauthor{Wenchao Du$^{\dagger}$}{1,2}
  \end{icmlauthorlist}

  \icmlaffiliation{1}{National Key Laboratory of Fundamental Science on Synthetic Vision, Sichuan University, Chengdu, China}
  \icmlaffiliation{2}{College of Computer Science, Sichuan University, Chengdu, China}

  \icmlcorrespondingauthor{Wenchao Du}{wenchaodu.cs@scu.edu.cn}


  \icmlkeywords{Gaze Estimation, Semi-Supervised Learning, Disentangled Contrastive Learning}
  
  \vskip 0.3in
]



\printAffiliationsAndNotice{}  

\begin{abstract}
Appearance-based gaze estimation always suffers from poor generalization due to limited annotated samples and insufficient dataset diversity. Leading approaches adopt weakly supervised learning to generate large-scale pseudo-labeled data from unconstrained real-world scenarios, aiming to mitigate the domain shifts. In this work, we devise a simple yet effective semi-supervised learning architecture that leverages unlabeled data to enhance domain generalization, thereby reducing reliance on labor-intensive manual annotations. Our key insight is to impose Jacobian regularization to disentangle feature representations into discriminative subspaces dedicated to specific gaze components, such as pitch and yaw angles. We further exploit the intrinsic ordinal ranking within each subspace for contrastive learning, enabling the model to learn robust gaze representations from a small set of labeled samples and an abundance of unlabeled ones. This ultimately yields our Disentangled Subspace Contrastive Learning (DSCL) framework. Extensive experiments on multiple benchmarks verify that the proposed DSCL is plug-and-play, achieving competitive performance using only 20\%, 10\%, and even 5\% of the annotated data under both in-domain and cross-domain evaluation settings. The public code is available at \href{https://github.com/da60266/DSCL}{https://github.com/da60266/DSCL}.

\end{abstract}

\section{Introduction}
Estimating accurate 3D gaze direction is crucial for understanding human behaviors and intentions, and provides substantial support for a wide range of practical applications \cite{stellmach2011designing,steil2018fixation,yi2022gazedock,lengenfelder2023pilot,bao2023exploring}, such as human–system interaction, mental fatigue detection, and AR/VR systems. Appearance-based gaze estimation \cite{chen2018appearance,cheng2022gazetr,wang2023gazecaps}, which directly predicts 3D gaze vectors from monocular images and videos, has attracted increasing attention in recent years. Nevertheless, constrained by the scarcity of high-quality annotated samples and the limited diversity of existing datasets, models trained on a single dataset often suffer from severe performance degradation when evaluated on unseen datasets. Variations in object appearance, background scenes, and image quality make it extremely challenging to learn robust gaze representations across domains \cite{xu2023learning}.

These challenges motivate recent advances on domain-adaptation-based and domain-generalization-based gaze estimation \cite{liu2021generalizing, bao2022generalizing,cheng2022puregaze,cai2023source,yin2024clip,bao2024feature}. The former requires additional samples from the target domain to fine-tune the model to adapt to a new data domain; the latter aims to learn a generalized gaze representation from a noisy distribution using only source-domain data. More recently, exploiting large-scale unlabeled data to assist learning generalized gaze representation has become a new trend. Leading works focus on extracting the weakly supervised signal from unlabeled samples to supervise gaze feature representation learning, e.g., mutual-gaze \cite{kothari2021weakly} and gaze-following \cite{vuillecard2025enhancing} from gaze-interaction scenarios, and developing a pseudo-annotation strategy. Moreover, unsupervised pretraining is also explored to disentangle gaze-related representation from facial appearance clues and then applied to the downstream gaze estimation \cite{bao2024unsupervised}. Though effective, these methods are still constrained by the limited diversity of the training sample and the number of high-quality labeled data.

In light of these limitations, exploiting unlabeled data to reduce reliance on high-quality annotated samples is both feasible and practical for real-world applications. Based on it, we extend semi-supervised learning (SSL) to generalized gaze estimation, where only a small fraction of data is labeled while the majority remains unlabeled. To this end, we propose a Disentangled Subspace Contrastive Learning (DSCL) framework that integrates the strengths of representation disentanglement and contrastive learning into a unified SSL paradigm. More specifically, DSCL first adopts Jacobian regularization to disentangle gaze representations for the regression of individual gaze components using limited labeled data. Subsequently, contrastive learning is conducted on these disentangled feature subspaces. We leverage the implicit label distance relationships among unlabeled samples and exploit the ordinal ranking derived from their feature similarity matrix to implement supervised contrastive learning, ultimately yielding robust gaze representations.

We provide theoretical analysis and empirical results demonstrating that our DSCL achieves competitive performance using only 20\%, 10\%, and even 5\% labeled data under both in-domain and cross-domain settings, compared with fully supervised training. This makes our method highly practical for real-world gaze estimation. To the best of our knowledge, DSCL is \textit{ the first to extend semi-supervised contrastive learning to leverage unlabeled data for gaze estimation.} In summary, our contributions are summarized: 1) Propose DSCL, a novel semi-supervised contrastive learning framework that pioneers the exploitation of unlabeled data for generalizable gaze estimation; 2) Integrate the merits of disentangling representation and contrastive learning into a unified SSL architecture, enabling more stable and robust gaze representation learning; 3) Proposed DSCL is plug-and-play, which achieves competitive performance in both in-domain and cross-domain gaze evaluation settings by utilizing merely 20\%, 10\%, and even 5\% labeled data, compared with state-of-the-art fully supervised methods.

\section{Related Work}
\subsection{Appearance-based Gaze Estimation}
Appearance-based gaze estimation \cite{cheng2024appearance} typically regresses gaze direction directly from facial images and has witnessed remarkable progress in recent years. Enabled by deep learning, these methods yield promising performance under in-domain evaluation settings. However, their success relies heavily on access to large-scale and diverse annotated training datasets. Collecting such data is notoriously expensive and labor-intensive, often requiring specialized hardware and substantial participant involvement \cite{zhang2017mpiigaze,fischer2018rt,kellnhofer2019gaze360}. Consequently, the limited scale and diversity of existing labeled data restrict model prediction accuracy and generalization capability, resulting in substantial performance degradation in cross-domain evaluations. To address this issue, domain generalization and domain adaptation methods have been extensively investigated to mitigate inherent domain shifts \cite{xu2023learning,liu2021generalizing,cai2023source,xu2024gaze}. More recently, extracting weakly supervised signals from unlabeled samples to facilitate generalizable gaze estimation has attracted growing interest. For instance, Kothari et al. \cite{kothari2021weakly} exploited geometric constraints from mutual gaze to generate pseudo 3D gaze labels. Vuillecard et al. \cite{vuillecard2025enhancing} leveraged gaze-following cues from gaze interaction scenarios to produce 3D pseudo labels, thereby boosting model generalization performance. Although these approaches achieve clear improvements in cross-domain evaluation, they still remain heavily dependent on high-quality labeled samples.

\subsection{Semi-Supervised Contrastive Learning}
Semi-supervised learning (SSL), which leverages a limited set of labeled samples and abundant unlabeled data to enhance model performance, has emerged as an effective paradigm for addressing the scarcity of labeled data in real-world applications \cite{yang2025umscs, cao2022open}. While initially focused on classification tasks \cite{berthelot2019mixmatch, sohn2020fixmatch}, SSL principles have been increasingly adapted for regression problems. For instance, Huang et al. \cite{huang2024rankup} proposed the RankUp framework to bridge the gap between semi-supervised classification and regression. Recently, contrastive learning has been incorporated into SSL to facilitate robust feature representation learning. By identifying positive and negative pairs based on their label relationships in the feature space, contrastive learning clusters similar samples together and provides an implicit supervisory signal for unlabeled data \cite{li2021comatch, yang2022class}. Dai et al. \cite{dai2023semi} were the first to propose the Semi-supervised Contrastive Learning (SCL) framework for solving general regression tasks with unlabeled data. Furthermore, state-of-the-art works have explored supervised contrastive learning to mitigate domain shift in gaze estimation \cite{yin2024clip, wang2022contrastive}. We also note a recent work, OMNIGAZE \cite{quomnigaze}, which claims to be the first to explore SSL for gaze estimation. However, it combined extremely large-scale labeled and unlabeled training samples and adopted a pseudo-labeling strategy based on multi-modal large models to provide additional supervision for unlabeled sample training, still falling under the category of weakly supervised learning frameworks. In contrast, our DSCL integrates disentangled representation learning into SCL, offering new insights for solving real-world gaze estimation tasks using only a small amount of labeled data.

\section{Methodology}
\subsection{Preliminaries}
To make this paper self-contained, we first briefly review the existing Semi-supervised Contrastive Learning for gaze estimation. Given a dataset $\mathcal{D}$ with $N$ samples, $\mathcal{D}:=\{ (x_i,y_i)\}_{i=1}^{L}$, where $x_i$ denote the input facial image, and $y_i$ is the corresponding gaze label, which is generally denoted by 2D Euler angle vector $(\phi, \psi)$ defined by pitch and yaw angles separately, and also expressed by a 3D gaze vector $(\mathrm{x}, \mathrm{y}, \mathrm{z})$, i.e., $\{y_i \in \mathbb{R}^{M} | M = 2,3\}$ is flexible according to different label settings. $\mathcal{D}':=\{x'_i\}^T_{i=1}$ denotes the unlabeled data consisting of input image only. The general appearance-based gaze estimation can be formulated as
\begin{equation}
z_i=E(x_i, \theta), \hat{y}_i = R(z_i, \beta),
\label{eq1}
\end{equation}
where the $E(\cdot)$ denotes an encoder, $R(\cdot)$ is the linear regressor, $\theta$ and $\beta$ are learnable parameters. $z$ is the encoded feature vector, which denotes the specific gaze representation, and $z'$ denotes feature for unlabeled data. The whole model can be trained through supervised learning with a mean square error (MSE) loss.

Contrastive learning aims to learn specific feature representations to reflect the label distance relationship in latent space, and is generally performed on normalized feature space. Thus, we denote the transformed $z_i$ as $\tilde{z}_i$ for the labeled sample. Let the supervised contrastive loss be $\mathcal{L}_{\textit{SC}} = \mathbb{E}_{x,x'} \| \langle E(x_j), E(x_i) \rangle - \mathcal{K}(y_i, y_j) \|^2$. The $\mathcal{K}$ denotes the kernel function to measure label distance, and $\langle \cdot, \cdot \rangle$ is the feature similarity function $\mathcal{S}$. The global minimum $E^*$ satisfies the isometry condition:
\begin{equation}
\begin{aligned}
	\left\langle E^*(x_i), E^*(x_j) \right\rangle &= \mathcal{K}(y_i, y_j), \\
	\text{s.t. } &(x_i, y_i), (x_j, y_j) \in \mathcal{D}.
	 \label{eq2}
\end{aligned}
\end{equation}
This implies that the geometry of the learned feature space $\mathcal{Z}$ is isomorphic to the label space $\mathcal{Y}$, which lets the feature with smaller label distance closer in the latent space. Previous work \cite{wang2022contrastive} has also proven $\mathcal{L}_{\textit{SC}}$ is useful for multi-target gaze estimation regression. We also provide the theoretical analysis of the consistency of supervised contrastive learning for multi-target regression in \cref{TA1}.

To exploit unlabeled data in semi-supervised settings, recent work \cite{dai2023semi} introduced an unsupervised contrastive constraint that refines feature representations via spectral ranking. However, this approach was fundamentally designed for single-target regression and cannot be directly applied to multi-target tasks like gaze estimation. The core limitation stems from an inherent dimensional mismatch: \textbf{\textit{The one-dimensional spectral rank inevitably collapses the target-specific ordinal relationships in a multidimensional label space, thereby introducing rank ambiguity.}}

\begin{proposition}
\label{TH1}
    Let $\mathcal{Y} \subseteq \mathbb{R}^M$ be a multi-target label space with $M \ge 2$. Let $\prec$ denote the strict component-wise partial order on $\mathcal{Y}$, where $y_i \prec y_j$ if and only if $y_{i,k} < y_{j,k}$ for all $k \in \{1, \dots, M\}$. There exists no scalar function $f: \mathcal{Y} \to \mathbb{R}$ that is strictly monotonic with respect to $\prec$ while preserving the independence of dimensions. Specifically, any scalar ranking $\mathcal{R}(z)$ introduces rank ambiguity, such that for distinct samples with conflicting dimensional changes (e.g., $\Delta y_{\cdot, 1} > 0, \Delta y_{\cdot, 2} < 0$), the scalar rank order is indeterminate and inconsistent with the true label space.
\end{proposition}

Detailed analysis is presented in \cref{TA2}. We demonstrate that this issue can be effectively addressed by constructing independent ranking orders for each target dimension. Motivated by this observation, we propose a novel DSCL framework. Specifically, we first leverage disentangled learning with Jacobian regularization to decouple latent representations into distinct subspaces. Subsequently, semi-supervised contrastive learning is conducted within each subspace to learn robust feature representations corresponding to individual gaze components. Further technical details are elaborated in subsequent sections.

\begin{figure*}
    \centering
    \includegraphics[width=\linewidth]{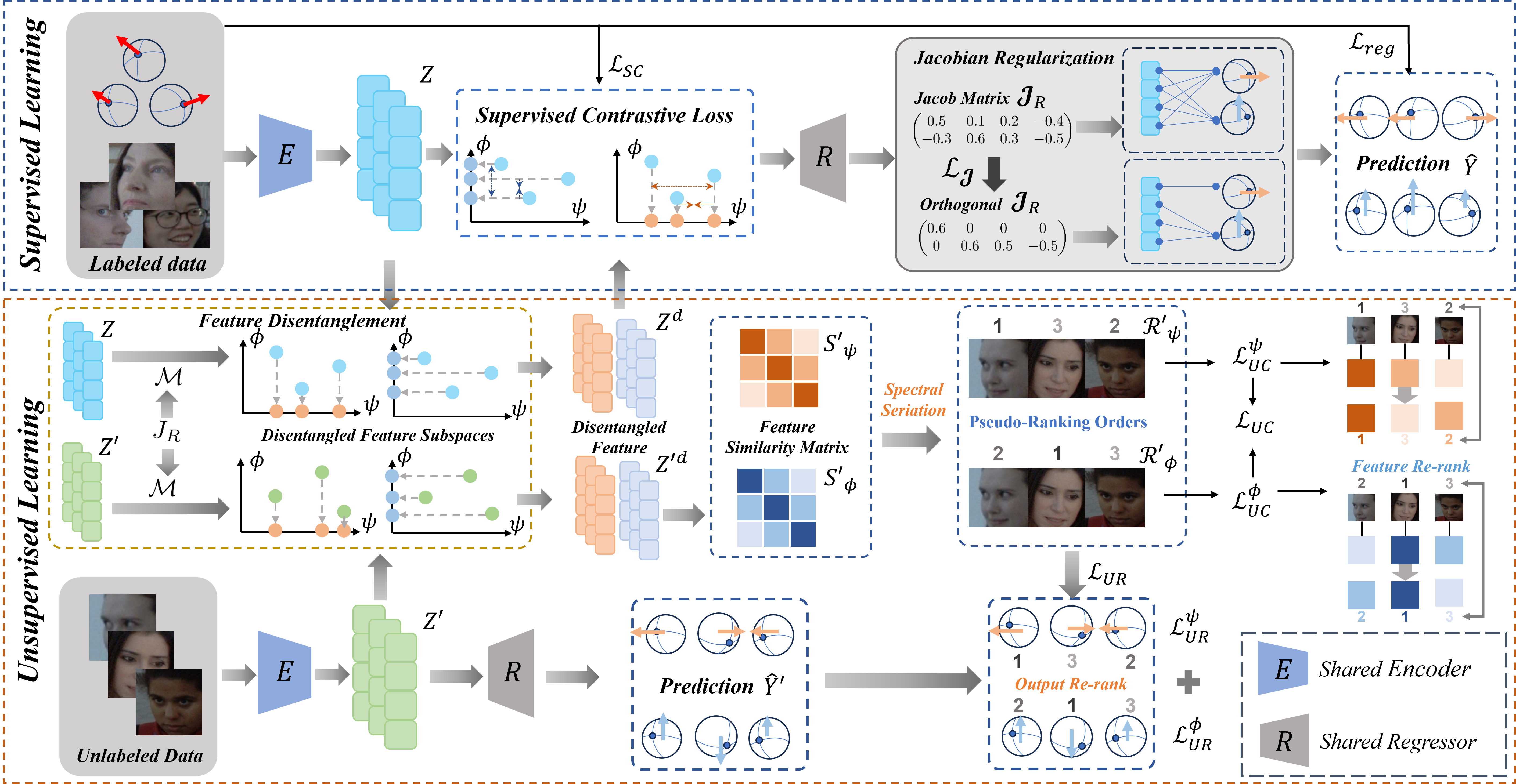}
    \caption{Overview of the Disentangled Subspace Contrastive Learning (DSCL) Framework. DSCL first disentangles the gaze representation $Z$ for specific gaze component (i.e., pitch $\phi$ and yaw $\psi$ angles) regression with Jacobian regularization $\mathcal{L}_{\pmb{J}}$ on labeled samples, and then performs unsupervised contrastive learning on each disentangled subspace with unlabeled data.}
    \label{fig2}
\end{figure*}
\subsection{Disentangling Independent Representation}
The proposed DSCL first explores the disentangled representation learning with labeled data. Suppose $Z\in\mathbb{R}^{B \times N}$ is the feature representation extracted from the encoder for samples with the batchsize of $B$ according to \cref{eq1}, where $N$ denotes its dimension, $\hat{Y}\in\mathbb{R}^{B \times M}$ is predicted gaze vector, the corresponding Jacobian matrix $\pmb{J}$ is expressed as:
\begin{equation}
	\pmb{J}_R(Z) =
	\begin{pmatrix}
		\frac{\partial R^1}{\partial Z^1} & \frac{\partial R^1}{\partial Z^2} & \cdots & \frac{\partial R^1}{\partial Z^N} \\
		\frac{\partial R^2}{\partial Z^1} & \frac{\partial R^2}{\partial Z^2} & \cdots & \frac{\partial R^2}{\partial Z^N} \\
		\cdots & \cdots & \cdots & \cdots \\
		\frac{\partial R^M}{\partial Z^1} & \frac{\partial R^M}{\partial Z^2} & \cdots & \frac{\partial R^M}{\partial Z^N} \\
	\end{pmatrix},
	\label{eq5-1}
\end{equation}
$R^m$ represents the regressor that maps the input $Z$ to the $m$-th gaze component output $\hat{Y}^m$. Each Jacobian element $\pmb{J}_{m,n} = \partial R^m / \partial Z^n$ denotes the sensitivity of the $m$-th output component to a change in the $n$-th feature dimension. In most practice, $M$ is set to 2 for regressing the Euler angles, i.e., pitch and yaw. It is obvious that each feature dimension typically relates to all the gaze components, leading to an entangled gaze representation.

To encourage disentangled representation learning, we constrain the Jacobian vectors across different dimensions to be orthogonal, thereby deriving a regularization loss,
\begin{equation}
	\mathcal{L}_{\pmb{J}} = \sum_{m=1}^M \sum_{z=1}^{M} \sum_{n=1}^{N} \mathbbm{1}_{m \neq z} |\pmb{J}_{m,n} \cdot \pmb{J}_{z,n}|,
    \label{eq5}
\end{equation}
where $\mathbbm{1}$ is the indicator function, equal to 1 if $m \neq z$ and 0 otherwise. This regularization enables each row of the Jacobian matrix to be mutually orthogonal, which results in a matrix where each column has only one non-zero element, meaning each feature dimension corresponds to only one gaze component.

\begin{theorem}
    Let $\pmb{J}$ be the Jacobian of the regressor. If $\pmb{J}$ satisfies the column-sparsity condition induced by minimizing $\mathcal{L}_{\pmb{J}}$ (i.e., for any column $n$, $\pmb{J}_{m,n} \cdot \pmb{J}_{k,n} = 0, \forall m \neq k$), then the feature space $\mathcal{Z} \cong \mathbb{R}^N$ decomposes into orthogonal subspaces:
\begin{equation}
\mathcal{Z} = \mathcal{Z}^1 \oplus \mathcal{Z}^2 \oplus \dots \oplus \mathcal{Z}^M,
\end{equation}
where each subspace $\mathcal{Z}^m$ exclusively controls the $m$-th target $\hat{y}_m$. That is, for any perturbation $\boldsymbol{\delta} \in \mathcal{Z}^m$, the change in output satisfies $\partial \hat{y}_k = 0$ for all $k \neq m$.
\end{theorem}

\begin{proof}
Let $\pmb{j}_m \in \mathbb{R}^N$ denote the $m$-th row vector of the Jacobian $\pmb{J}$, representing the gradient of the $m$-th target with respect to the features. The regularization in \cref{eq5} enforces the condition for $\forall m \neq k$:
\begin{equation}
\sum_{n=1}^N |\pmb{J}_{m,n} \cdot \pmb{J}_{k,n}| = 0 \implies
\text{supp}(\mathbf{j}_m) \cap \text{supp}(\mathbf{j}_k) = \emptyset,
\end{equation}
where $\text{supp}(\mathbf{v}) = \{i \mid v_i \neq 0\}$ is the support set of a vector.
Since the supports are disjoint, the set of feature indices $\{1, \dots, N\}$ can be partitioned into subsets $S_1, S_2, \dots, S_M$, where $S_m = \text{supp}(\mathbf{j}_m)$. The disjoint property implies $S_i \cap S_j = \emptyset$ for $i \neq j$.
\end{proof}
By applying Jacobian regularization on labeled data, the model tends to form a structured representation(i.e., disjoint latent subspaces), so that we can explicitly identify individual feature dimension for the specific gaze component by constructing a binary matrix $\mathcal{M} \in \mathbb{R}^{M\times N}$, note that only the $0$ and $1$ elements are contained in $\mathcal{M}$, where each row of $\mathcal{M}$ is used to isolate the feature dimension that primarily contributes to the specific gaze component. To build this binary matrix, we assign each feature dimension exclusively to the component corresponding to the largest absolute value. Formally, $\mathcal{M}$ is defined as:
\begin{equation}
	\mathcal{M}_{m,n} = \begin{cases}
		1, & \text{if } m = \underset{z}{\arg\max}|\pmb{J}_{z,n}|,\\
		0, & \text{otherwise}.
	\end{cases}
	\label{eq6}
\end{equation}

Subsequently, we can perform a point-wise multiplication between the binary matrix $\mathcal{M}$ and the feature matrix $Z$ to obtain the disentangled gaze representation for each sample,
\begin{equation}
	Z^d = Z \odot \mathcal{M}^T,
\end{equation}
where $Z^d \in \mathbb{R}^{B\times N \times M}$ denotes the disentangled feature matrix. We then split the $Z^d$ along with the batch axis to acquire a single disentangled feature vector $z^d \in N\times M$ for each sample. Note that for the $m$-th column in $z^d_i$ denotes the disentangled feature for $m$-th component corresponding to gaze label $y_i$.

\begin{remark}
A potential concern regarding the binary matrix $\mathcal{M}$ is the non-differentiability of the $\arg\max(\cdot)$ operator, which blocks gradient propagation through the mask generation path. However, this design is intentional and effective:
\begin{itemize}
    \item Forward Propagation: The matrix $\mathcal{M}$ acts as a gate, selecting specific feature dimensions for specific feature subspace ($Z^d = Z \odot \mathcal{M}^T$).
    \item Backward Propagation: During backpropagation, $\mathcal{M}$ functions as a constant filter. Gradients from the regression loss flow only through the selected active indices ($\mathcal{M}_{m,n}=1$). This means the regression loss focuses solely on refining the magnitude of the features within their assigned subspaces.
\end{itemize}
\end{remark}

\subsection{Subspace Contrastive Learning}
After disentangling gaze representation for a specific gaze component, we further construct a feature similarity matrix on $Z^d$. Here, we default that $Z^d$ is normalized for simplicity. We can obtain the feature distance in the disentangled subspace for each specific gaze component,
\begin{equation}
	D^m_z(z^d_i, z^d_j) = ||z^d_{i,m} - z^d_{j,m}||_2,
\end{equation}
where $D^m_z$ denotes the pair-wise feature distance within the subspace corresponding to the $m$-th gaze component. Similarly, we further apply the binary matrix $\mathcal{M}$, derived from the labeled data, to the feature representation extracted by the unlabeled data to enforce the consistent feature disentangling,
\begin{equation}
	D^m_z({z'}^d_i, {z'}^d_j) = ||{z'}^d_{i,m} - {z'}^d_{j,m}||_2.
\end{equation}

Therefore, we can obtain the feature similarity matrix set corresponding to each gaze component within a batch of labeled or unlabeled data, i.e., $\{\mathcal{S}_1, ..., \mathcal{S}_M\}$ and $\{\mathcal{S}'_1, ..., \mathcal{S}'_M\}$, and
\begin{equation}
	\begin{aligned}
		\mathcal{S}_m[i,j] &= D^m_z(z^d_{i}, z^d_{j}), \\
		\mathcal{S'}_m[i,j] &= D^m_z({z'}^d_{i}, {z'}^d_{j}).
	\end{aligned}
\end{equation}

After obtaining the set of feature similarity matrices, we then perform contrastive learning on each disentangled subspace. Under a semi-supervised learning setting, we implement supervised and unsupervised contrastive learning on labeled and unlabeled data simultaneously. By applying supervised contrastive loss $\mathcal{L}_{SC}$, we force the feature distances within each subspace to reflect the label discrepancies for a specific gaze component. Meanwhile, for unsupervised contrastive learning, we used spectral seriation \cite{atkins1998spectral} to generate a ranking order for each component, which provides the implicit supervision signal for unlabeled data.

Following \cite{dai2023semi}, we then exploit spectral seriation on the feature similarity matrix in each disentangled subspace, which allows us to derive ranking orders as the pseudo supervision signal. The seriation problem within the subspace can be formulated as:
\begin{equation}
	\underset{\mathcal{R'}_m}{\arg\min}\mathcal{S'}_m[i, j](\mathcal{R'}_m[i], \mathcal{R'}_m[j])^2,
\end{equation}
$\mathcal{R}'_m \in \mathbb{Z}^{B}$ is the ordinal ranking vector for samples in the unlabeled batch corresponding to the $m$-th component.

Consequently, the ranking order for each specific component is used to supervise contrastive learning on unlabeled data, and leads to an unsupervised contrastive learning loss $\mathcal{L}_{\textit{UC}}$,
\begin{equation}
	\begin{aligned}
		\mathcal{L}_{\text{UC}}
		&= \sum_{m=1}^{M} \mathcal{L}_{\text{UC}}^m(\mathcal{S}_m', \mathcal{R}_m') \\
		&= \sum_{m=1}^{M} \sum_{i=1}^{B} \ell \Big( \operatorname{rk}(\mathcal{S}_m'[i,:]), \\
		&\qquad \operatorname{rk}\big(-|\mathcal{R}_m' - \mathcal{R}_m'[i]|\big); \lambda \Big),
	\end{aligned}
\end{equation}
where $\operatorname{rk}(\cdot)$ denotes the ranking operator and $\ell(\cdot, \cdot)$ is the ranking similarity function. $\lambda$ is the hyperparameter for differential ranking solver \cite{poganvcic2019differentiation}. Similarly, the pseudo ranking order is also used to supervise each predicted gaze component $\hat{Y}'_m$, which results in an unsupervised ranking loss $\mathcal{L}_{\textit{UR}}$,
\begin{equation}
	\begin{aligned}
	    \mathcal{L}_{\textit{UR}} = &\sum_{m=1}^M \mathcal{L}^m_{\textit{UR}}(\hat Y'[:,m], \mathcal{R}'_m) \\
		= &\sum_{m=1}^{M} \sum_{i=1}^B \ell(\operatorname{rk}(-|\hat Y'[:,m] - \hat Y'[i,m]|), \\
        &\operatorname{rk}(-|\mathcal{R}'_m - \mathcal{R}'_m[i]|); \lambda). \\
	\end{aligned}
\end{equation}
\subsection{Joint Optimization}
The whole framework combines the supervised regression loss $\mathcal{L}_{\textit{reg}}$, Jacobian regularization loss $\mathcal{L}_{\pmb{J}}$, supervised contrastive loss $\mathcal{L}_{\textit{SC}}$, unsupervised contrastive learning loss $\mathcal{L}_{\textit{UC}}$, and unsupervised ranking loss $\mathcal{L}_{\textit{UR}}$,
\begin{equation}
	\begin{aligned}
		\mathcal{L}_{\textit{Total}} &= \mathcal{L}_{\textit{reg}}(Y,\hat Y) + \gamma \mathcal{L}_{\pmb{J}} + w_{\textit{SC}} \mathcal{L}_{\textit{SC}}\\
		&+ w_{\textit{UC}} \mathcal{L}_{\textit{UC}} + w_{\textit{UR}}\mathcal{L}_{\textit{UR}}
	\end{aligned}
\end{equation}
where $\gamma$, $w_{\textit{UC}}$, $w_{\textit{UR}}$ and $w_{\textit{SC}}$ are hyperparameters that balance the contribution of $\mathcal{L}_{\pmb{J}}$, $\mathcal{L}_{\textit{UC}}$, $\mathcal{L}_{\textit{UR}}$ and $\mathcal{L}_{\textit{SC}}$, respectively. Note that we implement $\mathcal{L}_{\textit{reg}}$ with the $L1$ norm to supervise gaze vector regression on labeled data. We employ supervised contrastive loss $\mathcal{L}_{\textit{SC}}$ as introduced in \cite{zhangimproving}. To further ensure training stability and reliability of the binary matrix $\mathcal{M}$, we first perform initialization on the labeled data using $\mathcal{L}_{\textit{reg}}$ and $\mathcal{L}_{\pmb{J}}$, and then finetune the model with $\mathcal{L}_{\textit{Total}}$ under the semi-supervised condition, which is crucial for establishing a well-structured feature space in which the Jacobian matrix exhibits clear column sparsity.
\begin{table*}[t]
	\caption{MAE comparison with state-of-the-art methods on the testing set of Gaze360.}
	\label{tab1}
    \centering
    \footnotesize
    \setlength{\tabcolsep}{8pt}
    \begin{tabular}{l|c|llll|l}
        \toprule
        Method & \#Params & Full & $20\%$ & $10\%$ & $5\%$ & Avg.\\
        \midrule
        Gazetr & 11.4M & 12.11$^{\circ}$ & 13.67$^{\circ}$ & 16.19$^{\circ}$ & 19.86$^{\circ}$ & 16.57$^{\circ}$\\
        Gazetr + DSCL & 11.4M & -- & 13.55$^{\circ}$ ($\downarrow 0.11^{\circ}$) & {14.36$^{\circ}$} ($\downarrow 1.83^{\circ}$) & {17.53$^{\circ}$} ($\downarrow 2.33^{\circ}$) & 15.14$^{\circ}$ ($\downarrow 1.43^{\circ}$)\\
        \midrule
        CA-Ne & 34.1M & 11.20$^{\circ}$ & -- & -- & -- & --\\
        L2CS-Net & 23.8M & 10.41$^{\circ}$ & 15.33$^{\circ}$ & 18.88$^{\circ}$ & 22.82$^{\circ}$ & 19.01$^{\circ}$ \\
        \midrule
        RT-Gene & 82.0M & 12.26$^{\circ}$ & 13.02$^{\circ}$ & 13.42$^{\circ}$ & 14.58$^{\circ}$  &  13.67$^{\circ}$\\
        RT-Gene + DSCL & 82.0M & -- & 12.83$^{\circ}$($\downarrow 0.19^{\circ}$) & 13.22$^{\circ}$($\downarrow 0.20^{\circ}$) &13.41$^{\circ}$($\downarrow 1.17^{\circ}$) & 13.15$^{\circ}$ ($\downarrow 0.52^{\circ}$)\\
        \midrule
        FullFace & 197M & 12.99$^{\circ}$ & {13.61$^{\circ}$} & {14.91$^{\circ}$} & 16.26$^{\circ}$ & 14.93$^{\circ}$ \\
        FullFace + DSCL & 197M & -- & {13.03$^{\circ}$}($\downarrow 0.58^{\circ}$)  & {14.81$^{\circ}$}($\downarrow 0.10^{\circ}$)  & {15.09$^{\circ}$}($\downarrow 1.17^{\circ}$)  & 14.31$^{\circ}$ ($\downarrow 0.62^{\circ}$) \\
        \midrule
        Dilated-Net & 3.90M & 13.73$^{\circ}$ & 13.83$^{\circ}$ & 15.04$^{\circ}$ & 16.11$^{\circ}$ & 14.99$^{\circ}$\\
        Dilated-Net + DSCL & 3.90M & -- & {13.57$^{\circ}$}($\downarrow 0.26^{\circ}$) & {14.83$^{\circ}$}($\downarrow 0.21^{\circ}$) & {15.55$^{\circ}$}($\downarrow 0.56^{\circ}$) & 14.65$^{\circ}$ ($\downarrow 0.34^{\circ}$)\\
        \midrule
        ST-WSGE & 27.5M & 11.58$^{\circ}$ & 13.71$^{\circ}$ & 16.20$^{\circ}$ & 21.33$^{\circ}$ & 17.08$^{\circ}$\\
        ST-WSGE + DSCL & 27.5M & -- & 13.05$^{\circ}$($\downarrow 0.66^{\circ}$) & 14.21$^{\circ}$($\downarrow 1.99^{\circ}$) & {17.10$^{\circ}$}($\downarrow 4.23^{\circ}$) & 14.79$^{\circ}$ ($\downarrow 2.29^{\circ}$)\\
        \midrule
        Baseline & 11.1M & 13.59$^{\circ}$ & 14.80$^{\circ}$ & 16.76$^{\circ}$ & 23.00$^{\circ}$ & 18.19$^{\circ}$ \\
        Baseline + DSCL & 11.1M & -- & 12.67$^{\circ}$($\downarrow 2.13^{\circ}$) & 13.94$^{\circ}$($\downarrow 2.82^{\circ}$) & 14.21$^{\circ}$($\downarrow 8.79^{\circ}$) & 13.60$^{\circ}$ ($\downarrow 4.59^{\circ}$) \\
        \bottomrule
    \end{tabular}
\end{table*}
\section{Experiments}
\subsection{Experimental Setting}
We select the three public benchmarks (i.e., Gaze360 \cite{kellnhofer2019gaze360}, MPIIGaze \cite{zhang2017mpiigaze}, and EyeDiap \cite{funes2014eyediap}) to evaluate the effectiveness of DSCL for both in-domain and cross-domain conditions. To simulate real data scarcity, we construct the few-label settings by uniformly sampling the Gaze360 dataset at ratios of 20\%, 10\%, and 5\% subsets, and serve as the labeled set, and the remainder of the data serves as the unlabeled set. For cross domain setting, the Gaze360 is used as the source domain, and MPIIGaze and EyeDiap as the target domain. As for in-domain evaluation, we only test on the Gaze360. Our baseline model used ResNet-18 \cite{he2016deep} as encoder, and a regressor with 2-layer MLPs to predict the gaze vector. The Mean Angular Error (MAE) is used as the common evaluation metric.

\subsection{Implementation Details}
All experiments are conducted on a single RTX 4090 GPU. Input images are resized to $224 \times 224$, following the common practice in prior work \cite{kellnhofer2019gaze360}. We used the Adam optimizer with a learning rate of $10^{-4}$ and a batch size of 32. The hyperparameter $\gamma$, $w_{\textit{SC}}$, $w_{\textit{UC}}$, and $w_{\textit{UR}}$ are empirically set to 1.0, 1.0, 0.05, 0.01, respectively, in all experiments. Specifically, $w_{\textit{SC}}$, $w_{\textit{UC}}$ and $w_{\textit{UR}}$ are set as same as \cite{dai2023semi}. For the domain-adaptation task, we train the model in the source domain (i.e., Gaze360) under different semi-supervised settings and perform unsupervised or supervised adaptation, reporting the average MAE over 20 trials using 100 samples randomly selected from the target domain.
\subsection{Experimental Results}
\subsubsection{Comparison with SOTA methods}
\textbf{In-domain Evaluation.} We first compare our DSCL with several representative gaze estimation methods \cite{cheng2022gazetr,cheng2020coarse,abdelrahman2023l2cs,fischer2018rt,zhang2017s,chen2018appearance,vuillecard2025enhancing} under the in-domain evaluation setting, with Gaze360 adopted as the benchmark dataset. Quantitative results are reported in \cref{tab1}. A clear trend can be observed that the performance of all methods heavily depends on the diversity and scale of annotated training data. Even ST-WSGE \cite{vuillecard2025enhancing}, which leverages weak supervision with additional unlabeled data, suffers from drastic performance degradation under limited labeled samples. Similarly, our baseline model exhibits the same tendency, with its MAE rising from 13.59$^{\circ}$ to 23.00$^{\circ}$. In contrast, when equipped with the proposed DSCL framework, the baseline achieves remarkable performance using only a small fraction of labeled data. It yields accuracy comparable to full-supervised training while introducing no extra computational overhead. Furthermore, embedding DSCL into existing baseline methods also brings consistent and stable performance gains across different settings.

\begin{table*}[t]
	\centering
	\footnotesize
	\caption{MAE comparison with the state-of-the-art methods for cross-domain evaluations. The $\mathcal{D}_G \rightarrow \mathcal{D}_M$ and $\mathcal{D}_G \rightarrow \mathcal{D}_E$ denote using Gaze360 as source domain for training, MPIIGaze and EyeDiap are viewed as target domains for evaluation, respectively. The baselines for Domain Generalization (DG) and Supervised Domain Adaptation (SDA) are indicated by $\dagger$ and $\ddagger$, respectively.}
	\label{tab2}
	\setlength{\tabcolsep}{0.8pt}
	\begin{tabular}{l|clll|clll}
		\toprule
		\multirow{2}{*}{Method} & \multicolumn{4}{c|}{$\mathcal{D}_G \rightarrow \mathcal{D}_M$} & \multicolumn{4}{c}{$\mathcal{D}_G \rightarrow \mathcal{D}_E$} \\
		  & Full & $20\%$ & $10\%$ & $5\%$ & Full &$20\%$ & $10\%$ & $5\%$ \\
		\midrule
		PureGaze & 9.23$^{\circ}$ & 11.02$^{\circ}$ & 14.68$^{\circ}$ & 16.10$^{\circ}$ & 9.32$^{\circ}$ & 16.05$^{\circ}$ & 18.61$^{\circ}$ & 20.82$^{\circ}$ \\
		PureGaze+DSCL & -- & 9.50$^{\circ}$ ($\downarrow 1.52^{\circ}$) & 10.57$^{\circ}$ ($\downarrow 4.11^{\circ}$) & 12.21$^{\circ}$ ($\downarrow 3.89^{\circ}$) & - & 9.32$^{\circ}$ ($\downarrow 6.73^{\circ}$) & 9.59$^{\circ}$ ($\downarrow 9.02^{\circ}$) & 14.40$^{\circ}$ ($\downarrow 6.42^{\circ}$) \\
		\midrule
		Gazecon & 8.52$^{\circ}$ & -- & -- & -- & 7.82$^{\circ}$ & -- & -- & -- \\
		CDG & 7.03$^{\circ}$ & -- & -- & -- & 7.27$^{\circ}$ & -- & -- & -- \\
		AGG & 7.87$^{\circ}$ & -- & -- & --& 7.93$^{\circ}$ & -- & -- & -- \\
		\midrule
		Baseline$\dagger$ &11.74$^{\circ}$ & 10.21$^{\circ}$ & 12.01$^{\circ}$ & 14.01$^{\circ}$ & 11.62$^{\circ}$ & 13.37$^{\circ}$ & 14.95$^{\circ}$ & 18.75$^{\circ}$ \\
		Baseline$\dagger$+DSCL & -- & 9.70$^{\circ}$($\downarrow 0.51^{\circ}$) & 9.89$^{\circ}$($\downarrow 2.12^{\circ}$) & 11.68$^{\circ}$ ($\downarrow 2.33^{\circ}$) & - & 9.01$^{\circ}$ ($\downarrow 4.36^{\circ}$) & 9.96$^{\circ}$ ($\downarrow 4.99^{\circ}$) & 10.21$^{\circ}$ ($\downarrow 8.54^{\circ}$) \\
		\midrule
		PnP-GA & 6.32$^{\circ}$ & 7.99$^{\circ}$ & 8.42$^{\circ}$ & 9.85$^{\circ}$ & 6.85$^{\circ}$&8.54$^{\circ}$ & 8.95$^{\circ}$ & 11.78$^{\circ}$ \\
		PnP-GA+DSCL & -- & 6.16$^{\circ}$ ($\downarrow 1.83^{\circ}$) & 6.55$^{\circ}$ ($\downarrow 1.87^{\circ}$) & 6.62$^{\circ}$ ($\downarrow 3.23^{\circ}$) & -- &6.76$^{\circ}$ ($\downarrow 1.78^{\circ}$) & 6.93$^{\circ}$ ($\downarrow 2.02^{\circ}$) & 7.58$^{\circ}$ ($\downarrow 4.20^{\circ}$) \\
		\midrule
		UnReGA & 5.80$^{\circ}$ & -- & -- & -- & 5.42$^{\circ}$ & -- & -- & -- \\
		UnReGA$^{-}$ & 6.21$^{\circ}$ &6.60$^{\circ}$ & 8.31$^{\circ}$ & 9.55$^{\circ}$ & 6.40$^{\circ}$ & 8.33$^{\circ}$ & 10.11$^{\circ}$ & 14.69$^{\circ}$ \\
		UnReGA$^{-}$+DSCL & -- & 6.18$^{\circ}$ ($\downarrow 0.42^{\circ}$) & 7.18$^{\circ}$ ($\downarrow 1.13^{\circ}$) & 9.22$^{\circ}$ ($\downarrow 0.33^{\circ}$) & -- & 6.91$^{\circ}$ ($\downarrow 1.42^{\circ}$) & 8.80$^{\circ}$ ($\downarrow 1.31^{\circ}$) & 11.52$^{\circ}$ ($\downarrow 3.17^{\circ}$) \\
		\midrule
		CRGA & 5.89$^{\circ}$ & -- & -- & -- & 6.49$^{\circ}$ & -- & -- & -- \\
		\midrule
		Baseline$\ddagger$ & 6.28$^{\circ}$ & 7.47$^{\circ}$ & 7.17$^{\circ}$ & 9.02$^{\circ}$ & 7.26$^{\circ}$ & 9.37$^{\circ}$ & 10.09$^{\circ}$ & 10.66$^{\circ}$ \\
		Baseline$\ddagger$+DSCL & -- & 6.63$^{\circ}$ ($\downarrow 0.84^{\circ}$) & 6.95$^{\circ}$ ($\downarrow 0.22^{\circ}$) & 7.22$^{\circ}$ ($\downarrow 1.80^{\circ}$) & -- & 8.15$^{\circ}$ ($\downarrow 1.22^{\circ}$) & 8.23$^{\circ}$ ($\downarrow 1.86^{\circ}$) & 8.64$^{\circ}$ ($\downarrow 2.02^{\circ}$) \\
		\bottomrule
	\end{tabular}
\end{table*}

\textbf{Cross-domain Evaluation.} We further conduct comprehensive comparisons under cross-domain settings. We select representative methods covering domain generalization (DG) \cite{cheng2022puregaze,xu2023learning,wang2022contrastive,bao2024feature}, unsupervised domain adaptation (UDA) \cite{liu2021generalizing,cai2023source,wang2022contrastive}, and supervised domain adaptation (SDA). Quantitative results are presented in \cref{tab2}. For DG-based methods, we observe that all approaches still yield high MAE values even when trained with fully labeled samples. By contrast, integrating our DSCL into the vanilla baseline enables our method to use merely 5\% labeled data, while achieving slightly better performance than the baseline trained on full labeled data. Furthermore, embedding DSCL into PureGaze also leads to substantial performance improvements. When extended to UDA and SDA methods, DSCL still achieves consistent and notable gains. These results demonstrate that our DSCL is both effective and efficient for learning generalizable gaze representations.

\begin{table}[t]
	\centering
	\footnotesize
	\caption{The ablation results of each loss term on the testing set of Gaze360 with different labeling rates.}
	\label{tab3}
	\setlength{\tabcolsep}{7pt}
	\begin{tabular}{l|ccc|c}
		\toprule
		Variant & $20\%$& $10\%$ & $5\%$& Avg. \\
		\midrule
		w/ $\mathcal{L}_{\textit{reg}}$ & 14.80$^{\circ}$ & 16.76$^{\circ}$ & 23.00$^{\circ}$ & 18.19$^{\circ}$ \\
		\midrule
		w/ $\mathcal{L}_{\textit{Total}}$ (DSCL) & \textbf{12.67$^{\circ}$} & \textbf{13.94$^{\circ}$} & \textbf{14.21$^{\circ}$} & \textbf{13.61$^{\circ}$} \\
		\midrule
		w/o $\mathcal{L}_{\textit{SC}}$ & 13.65$^{\circ}$ & 14.94$^{\circ}$ & 15.40$^{\circ}$ & 14.66$^{\circ}$ \\
		w/o $\mathcal{L}_{\textit{UR}}$ & 14.29$^{\circ}$ & 14.02$^{\circ}$ & 18.88$^{\circ}$ & 15.73$^{\circ}$ \\
		w/o $\mathcal{L}_{\textit{UC}}$ & 13.14$^{\circ}$ & 14.32$^{\circ}$ & 15.97$^{\circ}$ & 14.48$^{\circ}$ \\
		w/o $\mathcal{L}_{\boldmath{\pmb{J}}}$ & 13.68$^{\circ}$ & 14.81$^{\circ}$ & 17.66$^{\circ}$ & 15.38$^{\circ}$ \\
		w/o $\mathcal{L}_{\textit{UR}} + \mathcal{L}_{\textit{UC}}$ & 15.97$^{\circ}$ & 15.37$^{\circ}$ & 21.72$^{\circ}$ & 17.69$^{\circ}$
        \\
        w/o Init & 14.42$^{\circ}$ & 19.96$^{\circ}$ & 21.32$^{\circ}$ & 18.57$^{\circ}$\\
		\bottomrule
	\end{tabular}	
\end{table}
\subsubsection{Ablation Study}
\textbf{Loss Functions.} Since the proposed DSCL introduces no additional trainable parameters, we conduct ablation experiments solely on the Gaze360 dataset to analyze the contribution of each loss term. Quantitative results are listed in \cref{tab3}. We first train the baseline model in a supervised manner under the partially labeled data setting, denoted as $w/ \mathcal{L}_{reg}$. Compared with the full DSCL framework, the baseline exhibits poor generalization on the test set. We then gradually remove individual loss components from $\mathcal{L}_{\textit{Total}}$, and each ablated variant suffers from noticeable performance degradation to varying degrees. In particular, removing unlabeled data from the training set, i.e., w/o $\mathcal{L}_{\textit{UR}}$ and w/o $\mathcal{L}_{\textit{UC}}$, causes the most dramatic performance drop. This strongly verifies that DSCL effectively learns more generalizable gaze representations with the abundant unlabeled data. Additionally, eliminating the initial training phase (i.e., w/o Init) leads to a substantial drop in performance. This confirms that the supervised pretraining stage is essential for constructing a reliable and consistent binary matrix $\mathcal{M}$ that serves as the fundamental basis of our DSCL.

\begin{table}[t]
	\caption{Comparison with the state-of-the-art semi-supervised learning methods on the Gaze360. The best and second-best results are \textbf{bolded} and \underline{underlined}, respectively.}
	\label{tab4}
	\centering
	\footnotesize
	\setlength{\tabcolsep}{6pt}
	\begin{tabular}{l|cccc|c}
		\toprule
		Method  & Full & $20\%$ & $10\%$ & $5\%$ & Avg.\\
		\midrule
		\textit{Sup.} & 13.59$^{\circ}$ & \underline{14.80$^{\circ}$} & \underline{16.76$^{\circ}$} & 23.00$^{\circ}$ & 18.19$^{\circ}$ \\
		\midrule
		UCVME & - & 16.89$^{\circ}$ & 19.62$^{\circ}$ & 21.49$^{\circ}$ & 19.33$^{\circ}$ \\
		Rankup & - & 16.32$^{\circ}$ & 18.70$^{\circ}$ & 24.95$^{\circ}$  & 19.99$^{\circ}$ \\
		CLSS & - & 14.88$^{\circ}$ & 17.31$^{\circ}$ & \underline{20.77$^{\circ}$} & \underline{17.65$^{\circ}$} \\
		\midrule
		\textbf{DSCL} & - & \textbf{12.67$^{\circ}$} & \textbf{13.94$^{\circ}$} & \textbf{14.21$^{\circ}$} & \textbf{13.60$^{\circ}$} \\
		\bottomrule
	\end{tabular}
\end{table}
\begin{table}[t]
	\centering
	\caption{Scalability of DSCL with external unlabeled facial data. We report MAE under in-domain evaluation on Gaze360 and cross-domain evaluations from Gaze360 to EyeDiap ($\mathcal{D}_G \rightarrow \mathcal{D}E$) and MPIIGaze ($\mathcal{D}_G \rightarrow \mathcal{D}_M$). ``Baseline'' denotes using no external unlabeled data. The best and second-best results are highlighted with \textbf{bold} and \underline{underline}, respectively.}
	\label{tab5}
	\footnotesize
	\setlength{\tabcolsep}{7.5pt}
	\begin{tabular}{l|cccc}
		\toprule
		External data & Setting & 20\% & 10\% & 5\% \\  
		\midrule
		\multirow{3}{*}{Baseline} & Gaze360 & 12.67$^{\circ}$ & 13.94$^{\circ}$ & 14.21$^{\circ}$ \\
        & $\mathcal{D}_G \rightarrow \mathcal{D}_E$ & \textbf{9.01}$^{\circ}$ & 9.96$^{\circ}$ & \underline{10.21}$^{\circ}$ \\
        & $\mathcal{D}_G \rightarrow \mathcal{D}_M$ & 9.70$^{\circ}$ & 10.57$^{\circ}$ & 11.68$^{\circ}$ \\
        \midrule 
        \multirow{3}{*}{+WebFace} & Gaze360 & \underline{12.16}$^{\circ}$ & \underline{13.48}$^{\circ}$ & \underline{14.26}$^{\circ}$ \\
        & $\mathcal{D}_G \rightarrow \mathcal{D}_E$ & \underline{9.17}$^{\circ}$ & \underline{9.34}$^{\circ}$ & 11.23$^{\circ}$ \\
        & $\mathcal{D}_G \rightarrow \mathcal{D}_M$ & \underline{9.01}$^{\circ}$ & \underline{10.29}$^{\circ}$ & \underline{10.75}$^{\circ}$ \\
        \midrule
        \multirow{3}{*}{\shortstack{+WebFace \\ \& CelebA}} & Gaze360 & \textbf{11.76}$^{\circ}$ & \textbf{12.55}$^{\circ}$ & \textbf{13.85}$^{\circ}$ \\
        & $\mathcal{D}_G \rightarrow \mathcal{D}_E$ & \underline{9.03}$^{\circ}$ & \textbf{9.24}$^{\circ}$ & \textbf{10.09}$^{\circ}$ \\
        & $\mathcal{D}_G \rightarrow \mathcal{D}_M$ & \textbf{8.97}$^{\circ}$ & \textbf{9.15}$^{\circ}$ & \textbf{10.65}$^{\circ}$ \\
		\bottomrule
	\end{tabular}	
\end{table}

\textbf{Comparison with SOTA SSL Methods.} We further compare our DSCL with several representative semi-supervised learning frameworks to validate its effectiveness, e.g., UCVME \cite{dai2023ucvme}, Rankup \cite{huang2024rankup}, and CLSS \cite{dai2023semi}. All experiments are conducted on the Gaze360 dataset under in-domain settings, and quantitative results are reported in \cref{tab4}. For fair comparison, all the methods adopt our baseline architecture. The baseline is first trained with the 100\%, 20\%, 10\%, and 5\% training sets and then tests its performance, denoted by $\textit{Sup}$. Furthermore, we evaluate each approach to analyze their effects. It is clear that none of them achieved any improvement under semi-supervised settings, and, in turn, led to worse results, except for our DSCL. Specifically, compared to CLSS, our DSCL achieved significant performance gains, providing strong evidence for the existence of ambiguity in SCL. 

\textbf{Scaling with External Unlabeled Data.} To examine whether DSCL can benefit from external unlabeled data, we conduct an additional experiment by augmenting the unlabeled training pool with large-scale facial datasets. However, directly using external facial data may introduce a severe distribution mismatch due to differences in collection conditions and gaze distributions \cite{oliver2018realistic}. To better evaluate the scalability of DSCL while reducing the confounding effects caused by such a mismatch, we construct a mixed unlabeled pool by combining Gaze360 with external facial datasets at a ratio of $\sim 3:1$, which leads to the original data distribution remaining dominant.

Table~\ref{tab5} reports the results when progressively adding WebFace \cite{yi2014learning} and CelebA \cite{liu2015deep} into the unlabeled sample pool. Compared with the baseline without external data, adding WebFace already improves most evaluation settings, especially under the cross-domain setting $\mathcal{D}_G \rightarrow \mathcal{D}_M$. When further incorporating CelebA, DSCL achieves the best performance across all labeling rates and all evaluation settings. In particular, the MAE drops from 10.57$^{\circ}$ to 9.15$^{\circ}$ under $\mathcal{D}_G \rightarrow \mathcal{D}_M$ 10\% setting, implying that more diverse external unlabeled data can further improve both label efficiency and cross-domain generalization. These results demonstrate that DSCL can leverage the increased diversity of external facial data to learn more generalizable gaze representations and improve cross-domain robustness.

\textbf{Disentangling 3D Gaze Components.} Empirically, we adopt Euler angles as gaze labels, which consist of two components. We further extend the DSCL to 3D gaze vector prediction by setting \(M=3\) to explore its applicability. All experiments are conducted on the Gaze360, with quantitative results summarized in \cref{tab6}, where we also include CLSS for comparison. Our DSCL achieves superior performance under semi-supervised settings. In contrast, CLSS suffers from substantial performance drops. This indicates that three-dimensional regression introduces more severe ambiguity in unsupervised contrastive learning, which greatly restricts its representation learning capability.

\begin{table}[t]
	\caption{Ablation study for disentangling more gaze components.}
	\label{tab6}
	\centering
	\footnotesize
	\setlength{\tabcolsep}{3.8pt}
	\begin{tabular}{l|c|ccc|c}
		\toprule
		Method & Label & $20\%$ & $10\%$ & $5\%$ & Avg. \\
		\midrule
		\multirow{2}{*}{CLSS} & $M=2$ & 14.88$^{\circ}$ & 17.31$^{\circ}$ & 20.77$^{\circ}$ & 17.65$^{\circ}$ \\
		& $M=3$ & 15.53$^{\circ}$ & 18.44$^{\circ}$ & 22.88$^{\circ}$ & 18.95$^{\circ}$ (+1.30$^{\circ}$) \\
		\midrule
		\multirow{2}{*}{\textbf{DSCL} } & $M=2$ & 12.67$^{\circ}$ & 13.94$^{\circ}$ & 14.21$^{\circ}$ & 13.61$^{\circ}$ \\
		& $M=3$ & 12.29$^{\circ}$ & 13.32$^{\circ}$ & 14.86$^{\circ}$ & 13.49$^{\circ}$ (- 0.12$^{\circ}$) \\
		\bottomrule
	\end{tabular}
\end{table}
\begin{figure}[t]
	\centering
	\includegraphics[width=\linewidth]{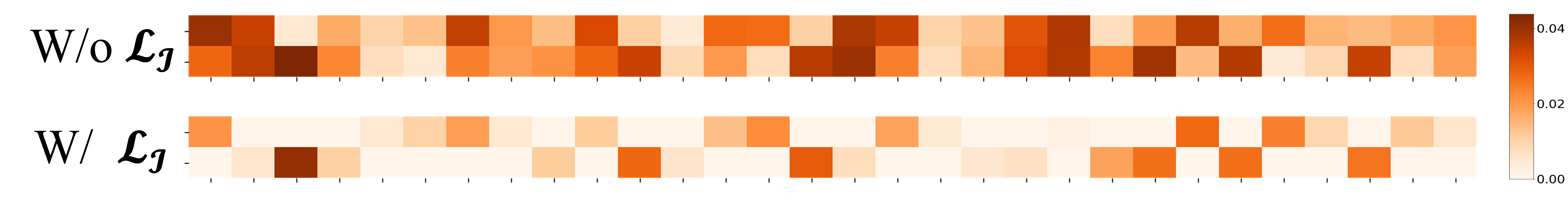}
	\caption{Visualization for regularization Jacobian matrix. Darker colors indicate larger absolute values, while lighter colors represent smaller ones.}
	\label{fig3}
\end{figure}

\textbf{Jacobian Regularization.} We also visualize the regularized Jacobian matrix with $\mathcal{L}_{\pmb{J}}$ in \cref{fig3} to validate the effectiveness of the DSCL. Removing the Jacobian regularization constraint (i.e., w/o $\mathcal{L}_{\pmb{J}}$) exhibits a clear vertical alignment pattern, the deeper colors consistently align across both output rows, signifying a coupled feature embedding, as both gaze components are sensitive to the same feature dimensions. In contrast, using $\mathcal{L}_{\pmb{J}}$ successfully disentangles this correlation. This provides direct visual evidence for our insight that forces the model to learn a more disentangled representation for each gaze component.

\begin{figure}[t]
    \centering
    \begin{minipage}[t]{0.46\linewidth}
        \centering
        \includegraphics[width=\linewidth]{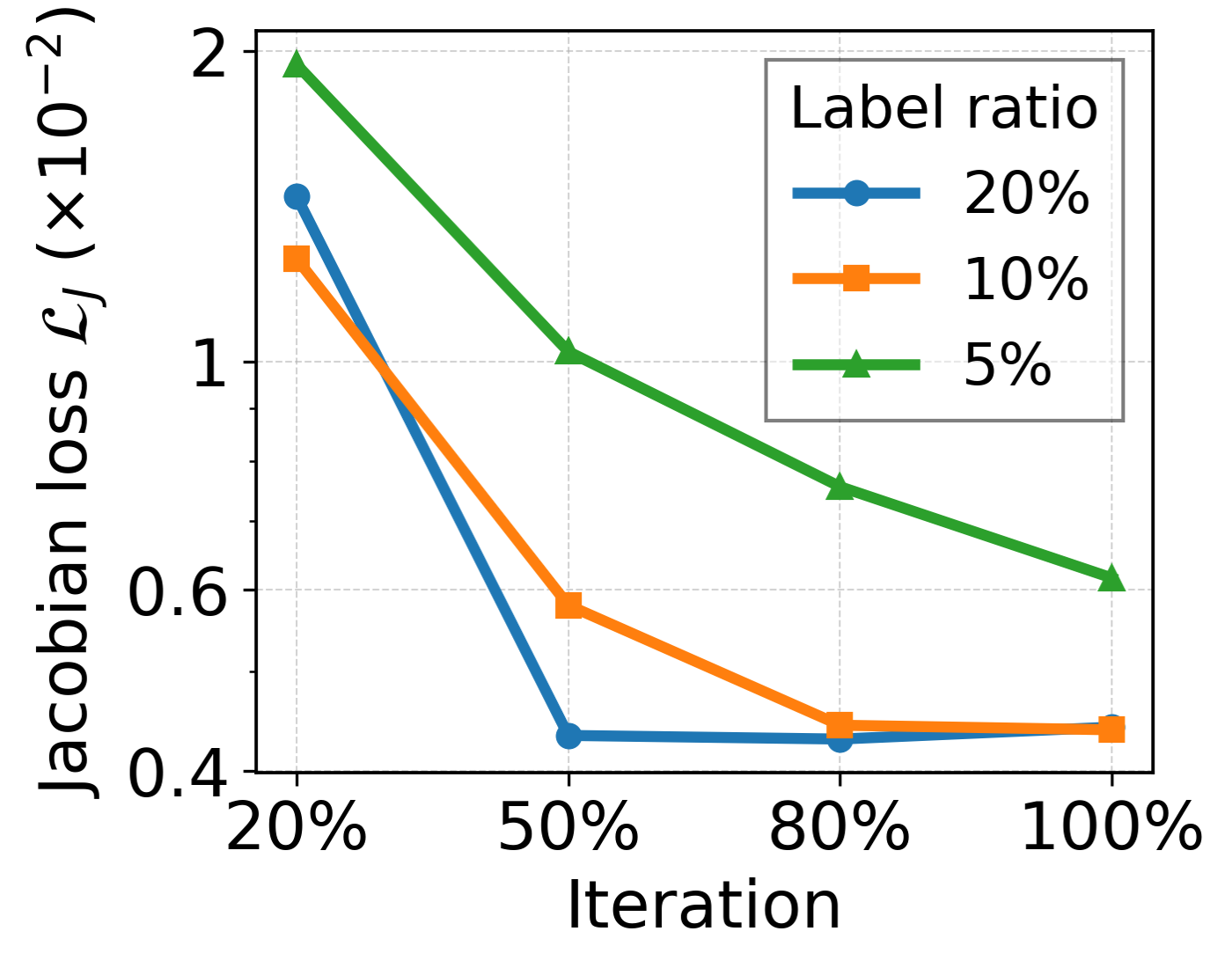}
        \centerline{\small (a) $\mathcal{L}_{\pmb{J}}$ in Pretraining}
    \end{minipage}
    \hfill
    \begin{minipage}[t]{0.48\linewidth}
        \centering
        \includegraphics[width=\linewidth]{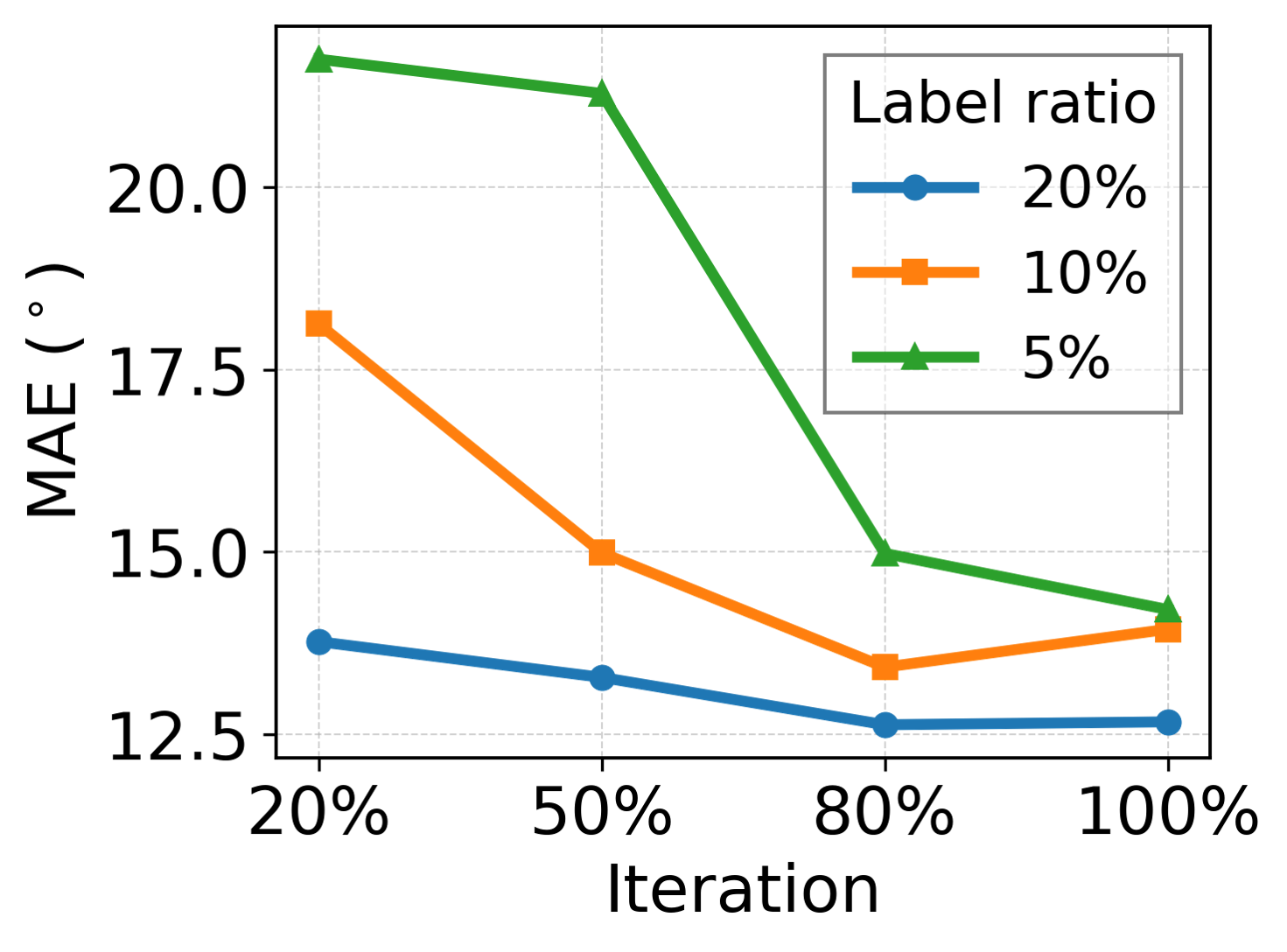}
        \centerline{\small (b) Training Error}
    \end{minipage}
    \caption{Sensitivity Analysis for Pre-trained \(\mathcal{M}\). (a) The Jacobian regularization loss \(\mathcal{L}_{\pmb{J}}\) rapidly decreases and tends to stabilize by around 50\% of iterations in the initialization phase. (b) Semi-supervised training errors significantly drop as the mask \(\mathcal{M}\) becomes stable.
    }
    \label{fig4}
\end{figure}

\textbf{Sensitivity Analysis for Pre-trained \(\mathcal{M}\).} We further examine whether the first phase can efficiently construct a reliable mask matrix $\mathcal{M}$. As shown in \cref{fig4}(a), the Jacobian loss $\mathcal{L}_{\pmb{J}}$ rapidly decreases and tends to stabilize by around 50\% of the initialization progress, indicating fast convergence of the disentangled mask. \cref{fig4}(b) shows that the corresponding training errors quickly decrease as $\mathcal{M}$ becomes stable, and then enter a stable performance plateau. This verifies that the initialization phase is lightweight and robust, providing reliable masks for subspace disentangling.

\begin{figure}[t]
	\centering
	\includegraphics[width=\linewidth]{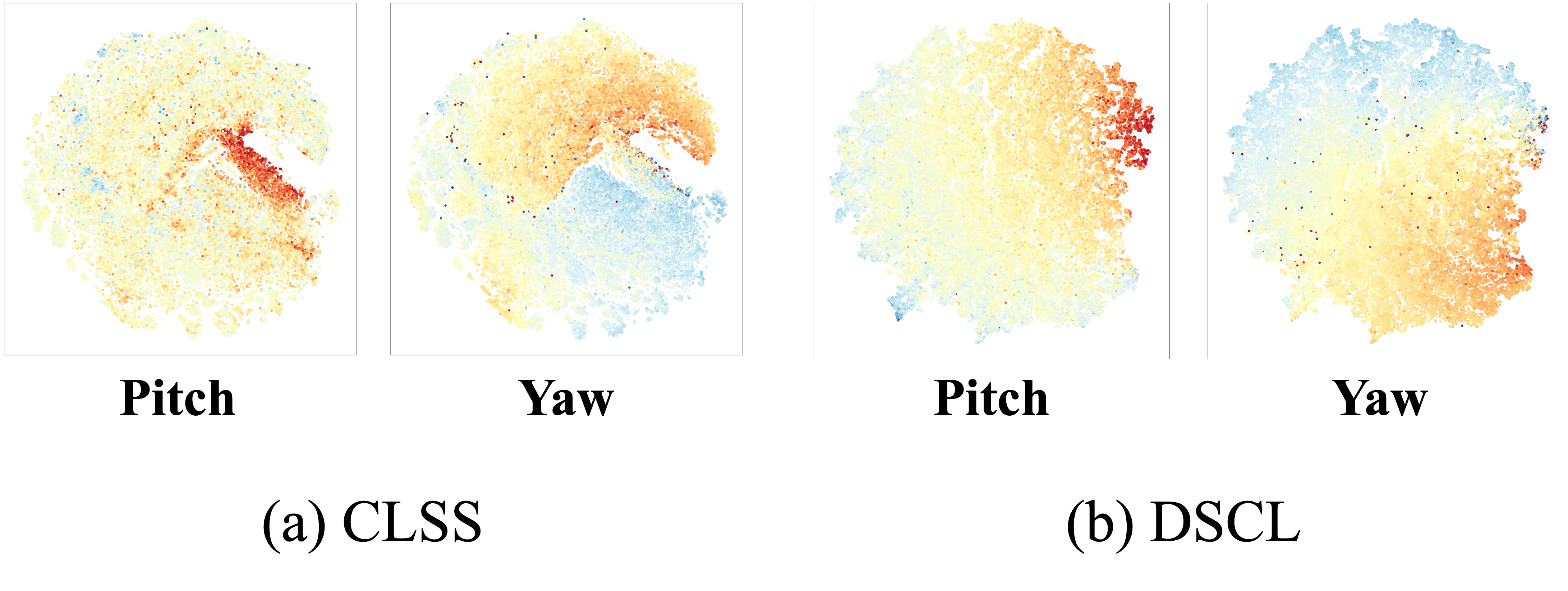}
	\caption{Visualization of the gaze feature distribution for unlabeled data. Different colors denote different gaze directions and close gaze direction share similar colors.}
	\label{fig5}
\end{figure}

\begin{figure}[t]
	\centering
	\includegraphics[width=\linewidth]{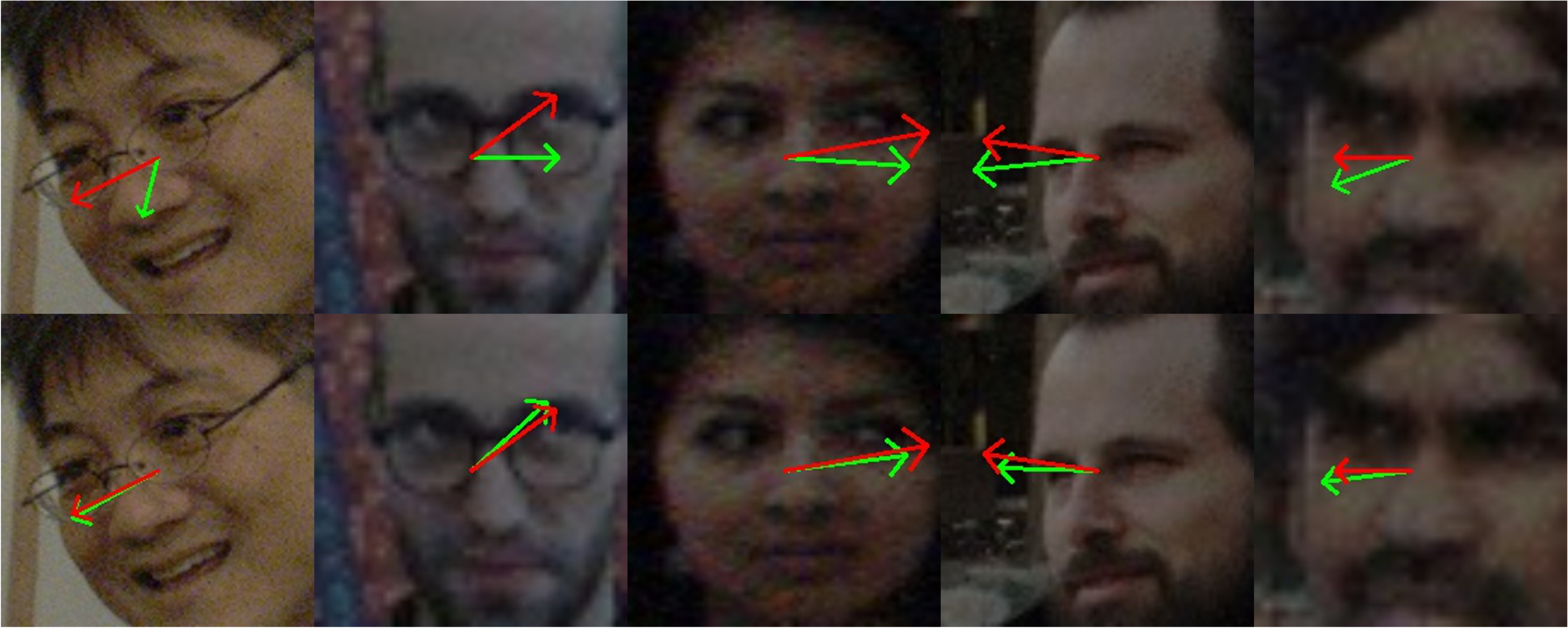}
	\caption{Visualization results of the predicted gaze for unlabeled data. The first row is from CLSS, and the bottom row is DSCL. The red and green arrows denote the ground truth and predicted gaze vector, respectively.}
	\label{fig6}
\end{figure}

\textbf{Disentangled Representation Visualization.} To better illustrate the resolution of rank ambiguity in unlabeled training samples, we first visualize the disentangled gaze representations using $t$-SNE in \cref{fig5}, where points with similar gaze direction are colored similarly. CLSS \cite{dai2023semi} produces severely mixed representations for pitch and only slightly clearer clusters for yaw. Conversely, our DSCL generates a well-ordered feature space with distinct clusters that strongly correlate with true gaze directions.

Furthermore, we visualize the prediction on the label space. As shown in \cref{fig6}, our DSCL accurately disentangles gaze directions with the unsupervised learning, which demonstrates that our method effectively overcomes the rank ambiguity in multidimensional continuous spaces. By preserving the ordinal consistency of target labels, DSCL enables highly accurate gaze estimation for unlabeled samples without suffering from dimensional conflicts.

\section{Conclusion}
This paper proposes a novel Disentangled Subspace Contrastive Learning (DSCL) framework for semi-supervised gaze estimation, which effectively reduces dependence on labor-intensive manual annotations and enables more efficient exploitation of unlabeled data. By seamlessly combining the advantages of disentangled learning and semi-supervised contrastive learning within a unified paradigm, DSCL alleviates the inherent ambiguity existing in unsupervised contrastive learning, thus facilitating robust gaze prediction. Moreover, DSCL features strong plug-and-play applicability and achieves performance comparable to fully supervised methods under both in-domain and cross-domain scenarios. This work also offers a new insight to advance practical, real-world gaze estimation tasks.

\section*{Impact Statement}
DSCL reduces the reliance of appearance-based gaze estimation on large-scale annotations by leveraging unlabeled data in a semi-supervised regression framework. Methodologically, it addresses rank ambiguity in semi-supervised contrastive regression for multi-dimensional outputs by disentangling gaze representations into component-specific subspaces and constructing ordinal relationships within each subspace. This enables unlabeled data to provide more reliable supervision for continuous gaze prediction. These advances may benefit applications where gaze labels are costly to collect, such as human–computer interaction, AR/VR, assistive systems, driver monitoring, and behavioral analysis, while improving robustness under low-label and cross-domain settings. 



\bibliography{ref}
\bibliographystyle{icml2026}

\newpage
\appendix
\crefalias{section}{appendix}
\crefalias{subsection}{appendix}
\onecolumn
\section{Theoretical Analysis}
\subsection{Supervised Contrastive Learning Consistency for Multi-target Regression}
\label{TA1}
Consider a multi-target regression setting where the label is a vector ${y} \in \mathbb{R}^M$, representing $M$ distinct regression values. Let the feature encoder be $E: \mathcal{X} \to \mathcal{Z}$ where $\mathcal{Z}$ is the unit hypersphere, since feature has been normalized. We can reformulate the Supervised Contrastive Regression loss $\mathcal{L}_{SC}$ over a batch of $N$ samples as:
\begin{equation}
    \mathcal{L}_{SC} = \sum_{i=1}^N \sum_{j=1}^N \left\| \langle E(x_i), E(x_j) \rangle - \mathcal{K}(y_i, y_j) \right\|^2_F
\end{equation}
where $\mathcal{K} : \mathbb{R}^M \times \mathbb{R}^M \rightarrow \mathbb{R}$ is a positive semi-definite kernel function measuring similarity in the multi-target label space.
The global minimum $E^*$ induces a feature space geometry that satisfies:
\begin{equation}
    G^\mathcal{Z}_{ij} = G^\mathcal{Y}_{ij}, \quad \forall i,j \in \{1, \dots, N\}
\end{equation}
where $G^\mathcal{Z}$ and $G^\mathcal{Y}$ are the Gram matrices of the features and the multi-target labels, respectively. This implies that the learned feature similarities are encouraged to match the kernel-induced geometry of the multi-target label space.

Let $S_{ij} = \langle z_i, z_j \rangle$ denote the cosine similarity between learned features, and let $T_{ij} = \mathcal{K}({y}_i, {y}_j)$ denote the ground-truth similarity derived from the multi-dimensional label vectors. The loss function can be rewritten in terms of the Frobenius norm between the two Gram matrices:
\begin{equation}
    \mathcal{L}_{SC} = \| \mathbf{S} - \mathbf{T} \|_F^2 = \sum_{i,j} (S_{ij} - T_{ij})^2
\end{equation}
To find the stationary point of the encoder, we analyze the gradient of the loss with respect to the similarity structure $S_{ij}$. The derivative is given by:
\begin{equation}
    \frac{\partial \mathcal{L}_{SC}}{\partial S_{ij}} = 2(S_{ij} - T_{ij})
\end{equation}
Setting the gradient to zero for optimality condition $\nabla \mathcal{L}_{SC} = 0$, we obtain the strictly enforced equality $S_{ij} = T_{ij}$.

\textbf{Consistency Mechanism for Multiple Targets:}
Crucially, even though $y_i$ is a $M$-dimensional vector, the kernel $\mathcal{K}(\cdot, \cdot)$ acts as a scalarization function that aggregates the geometric relationships across all $M$ dimensions. For instance, if we use a Gaussian RBF kernel $\mathcal{K}(y_i, y_j) = \exp(-\gamma \|y_i - y_j\|_2^2)$, the term $\|y_i - y_j\|_2^2$ represents the Euclidean distance in the joint label space $\mathbb{R}^M$.

Consequently, $T_{ij}$ inherently encodes the \textit{joint distribution} and the inter-correlations of the multiple targets. By forcing $S_{ij} \to T_{ij}$, the optimization process compels the feature encoder $f(\cdot)$ to learn a representation where the distance between $z_i$ and $z_j$ strictly reflects the composite distance in the multi-target label space. 
Therefore, supervised contrastive learning remains effective for multi-target regression because it aligns the topology of the feature manifold with the joint topology of the vector-valued labels, implicitly capturing the dependencies between different target dimensions without requiring separate loss functions for each target.

\subsection{Rank Ambiguity of Scalar Ranking in Multi-Target Regression}
\label{TA2}

Assume there exists a scalar ranking function $rk: \mathbb{R}^d \to \mathbb{R}$ that perfectly preserves the ordering of all target dimensions. For the function to be valid for regression refinement, it must satisfy strict monotonicity for each dimension independently: $\partial rk / \partial y_k > 0, \quad \forall k \in \{1, \dots, M\}$. Consider two samples $A$ and $B$ in a gaze estimation setting ($M=2$) with labels $y_A$ and $y_B$. Let $y_A$ be a base reference, and $y_B$ differ from $y_A$ by a positive increment in the pitch and a negative increment in the yaw:
\begin{equation}
    y_B = y_A + (\delta_1, -\delta_2), \quad \delta_1, \delta_2 > 0.
\end{equation}
The change in the scalar rank, denoted as $\Delta \mathcal{R} = rk(y_B) - rk(y_A)$, can be approximated by the total differential:$\Delta \mathcal{R} \approx \frac{\partial rk}{\partial y_1}\delta_1 + \frac{\partial rk}{\partial y_2}(-\delta_2)$. Since both partial derivatives are positive, the term $\frac{\partial rk}{\partial y_1}\delta_1$ is positive, while $\frac{\partial rk}{\partial y_2}(-\delta_2)$ is negative. Consequently, the sign of $\Delta \mathcal{R}$ depends entirely on the magnitude of the partial derivatives and the perturbations $\delta$.
    
\textbf{Case 1:} If $\Delta \mathcal{R} > 0$, the ranking implies $y_B$ is "greater" than $y_A$, contradicting the fact that $y_{B,2} < y_{A,2}$.
    
\textbf{Case 2:} If $\Delta \mathcal{R} < 0$, the ranking implies $y_B$ is "smaller" than $y_A$, contradicting the fact that $y_{B,1} > y_{A,1}$.
    
\textbf{Case 3:} If $\Delta \mathcal{R} = 0$, the distinct samples are mapped to the same rank, causing feature collapse.
    
In an unsupervised setting, the gradients $\frac{\partial rk}{\partial y_k}$ are unknown and effectively random, determined by the spectral properties of the data manifold rather than semantic meaning. Thus, the scalar rank cannot structurally enforce the component-wise order of $\mathcal{Y}$.

\section{Gaze Estimation Datasets}

In this section, we provide detailed descriptions of gaze estimation benchmarks used in the experiments.

\textbf{Gaze360} \cite{kellnhofer2019gaze360} was collected in both indoor and outdoor environments, comprising labeled 3D gaze data from 238 subjects with a diverse range of head poses and gaze directions. Following the preprocessing steps in \cite{cheng2024appearance}, we exclude images where the subject's face is not visible. The remaining data is divided into a training set of 84,902 images, which is further split into labeled and unlabeled subsets. The gaze distributions of labeled under various semi-supervised settings are illustrated in \cref{appen_fig1}. And 16,031 images are served as the test set for in-domain evaluations.

\textbf{MPIIGaze} \cite{zhang2017mpiigaze} was collected from 15 subjects in unconstrained real-world environments. Adhering to the standard evaluation protocol, we select a subset of 3,000 images from each subject. For cross-domain evaluations, consistent with previous works \cite{cheng2022puregaze, liu2021generalizing}, we utilize these 45,000 images as the target domain.

\textbf{EyeDiap} \cite{funes2014eyediap} consists of video clips recorded from 16 subjects, where gaze targets are defined by either screen targets or 3D floating balls. Adhering to the protocol in \cite{cheng2024appearance}, we select 16,674 images from 14 subjects to serve as the target domain for cross-domain evaluations.

\begin{figure}
    \centering
    \includegraphics[width=\linewidth]{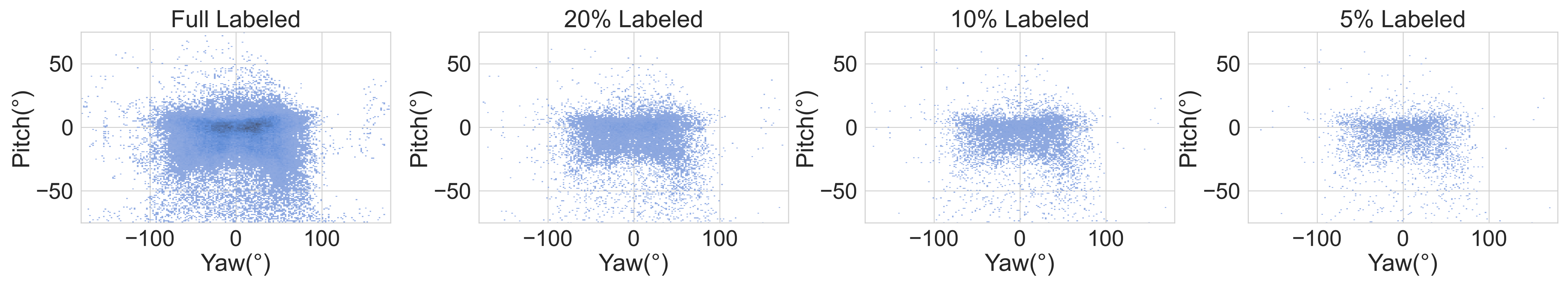}
    \caption{Sampled gaze label distribution under different semi-supervised settings.}
    \label{appen_fig1}
\end{figure}

\section{Extension to Other Multi-Target Regression Tasks}
Beyond the gaze estimation, we further conduct experiments to compare DSCL and CLSS on two classical multi-target regression tasks.

\subsection{Datasets}
\textbf{Sarcos} \cite{seeger2004gaussian}: An inverse dynamics problem for a seven degree-of-freedom SARCOS anthropomorphic robot arm. This task is to map a 21-dimensional input space to corresponding 7 joint torques, which is always treated as a multi-target regression problem.

\textbf{Dsprites}\cite{dsprites17}: a synthetic 2D dataset generated from six ground truth independent latent factors, including shape, scale, orientation, and position. The primary task is to benchmark the disentangled ability of unsupervised learning methods. We adopt it to predict latent factors, e.g., position X, Y and scale.

\subsection{Implementation Details} 
For tabular dataset, i.e., Sarcos, we use a Feed Forward Neural Network (FFNN) with hidden neuron $\{64,128,256,512 \}$ as the backbone model to extract latent feature. For Dsprites, we use ResNet-18 as the backbone. The total training epoch is set as 100 and batch size is set as 16. We utilize Adam optimizer with learning rate $2e^{-5}$ during the training. The hyperparameters $\gamma$, $w_{\textit{SC}}$, $w_{\textit{UC}}$ and $w_{\textit{UR}}$ are set to 1.0, 1.0, 0.05, 0.01. We randomly select samples for Dsprites dataset, and split $20\%$ data as test set, the remaining 80\% as training set. To conduct semi-supervised learning, we further randomly select labeled training data with specific ratio or number, and the remaining data serve as unlabeled training data. The statistics for datasets are summarized in \cref{tab7}. The performance is evaluated with four metrics: Mean Absolute Error (MAE), Root Mean Squared Error (RMSE), Pearson correlation coefficient (Pearson), Spearman's rank correlation coefficient (Spearman).

\subsection{Experimental Results}
The testing results are shown in \cref{tab8}. Compared to CLSS, the proposed DSCL framework achieves superior performances across multiple metrics.

\begin{table}[t]
  \caption{Detailed Statistics of Datasets}
  \centering
        \begin{tabular}{lccc}
          \toprule
          Dataset & \# of instance & \# of feature & \# of regression target $(M)$ \\
          \midrule
          Sarcos & 44,484 & 21 & 7\\
          Dsprites & 20,000 & 64$\times$64 & 3 \\
          \bottomrule
        \end{tabular}
  \label{tab7}
\end{table}

\begin{table}[t]
  \caption{Testing performance results on different tasks under semi-supervised setting. Mean and standard deviation are reported. $|X_l|$ denotes the ratio or number of labeled training data. $\downarrow$ means the smaller the better, and $\uparrow$ means the greater the better. The best results are \textbf{bolded}.}
  \footnotesize
  \setlength{\tabcolsep}{14pt}
  \centering
    \begin{tabular}{l|llllll}
      \toprule
      Method & Dataset & $|X_l|$ & MAE$\downarrow$ & RMSE$\downarrow$ & Pearson$\uparrow$ & Spearman$\uparrow$\\
      \midrule
      \multirow{6}{*}{\textit{Sup.}} & \multirow{3}{*}{Sarcos} & 10\% & 1.7147 (0.018) & 2.461 (0.019) & 0.9438 (0.001) & 0.9334 (0.001)\\
      & & 5\% & 2.1572 (0.028) & 2.998 (0.023)  & 0.9235 (0.001) & 0.9107 (0.001)\\
      & & 1\% & 3.6324 (0.130)  & 4.930 (0.202)  & 0.8012 (0.012)  & 0.7777 (0.009)\\
      \cmidrule{2-7}
      & \multirow{3}{*}{Dsprites} & 300 & 0.0770 (0.006) & 0.092 (0.007) & 0.9064 (0.012) &0.9122 (0.011) \\
      & & 100 & 0.0944 (0.004) & 0.119 (0.004)  & 0.8570 (0.014)  &0.8540 (0.013) \\
      & & 50 & 0.1067 (0.006) & \textbf{0.135 (0.007)}  & 0.7736 (0.008) & 0.7766 (0.008)\\
      \midrule
      \multirow{6}{*}{CLSS} & \multirow{3}{*}{Sarcos} & 10\% & 1.7637 (0.094)  & 2.528 (0.122)  & 0.9418 (0.003) & 0.9318 (0.004)\\
      & & 5\% & 2.1269 (0.009) & 2.962 (0.016)  & 0.9240 (0.003) & 0.9105 (0.003) \\
      & & 1\% & 3.0046 (0.018)   & 4.041 (0.058)  & \textbf{0.8762} (0.002)  & \textbf{0.8594 (0.001)}\\
      \cmidrule{2-7}
      & \multirow{3}{*}{Dsprites} & 300 & 0.1501 (0.017)  & 0.192 (0.021)  & 0.7316 (0.046) & 0.7345 (0.052)  \\
      & & 100 & 0.1638 (0.010)  & 0.207 (0.014)   & 0.6931 (0.025)  & 0.7009 (0.022) \\
      & & 50 & 0.2507 (0.021)  & 0.315 (0.023)   & 0.4268 (0.141) & 0.4329 (0.149)\\
      \midrule
      \multirow{6}{*}{\textbf{DSCL}} & \multirow{3}{*}{Sarcos} & 10\% & \textbf{1.4003 (0.006)} & \textbf{2.061 (0.013)}  & \textbf{0.9548 (0.003)} & \textbf{0.9437 (0.002)} \\
      & & 5\% & \textbf{1.7900 (0.017)}  & \textbf{2.577 (0.021)}  & \textbf{0.9295 (0.006)} & \textbf{0.9159 (0.008)} \\
      & & 1\% & \textbf{2.8742 (0.071)}  & \textbf{3.927 (0.067)} & 0.8710 (0.007) & 0.8521 (0.010)\\
      \cmidrule{2-7}
      & \multirow{3}{*}{Dsprites} & 300 & \textbf{0.0695 (0.008)} & \textbf{0.091 (0.007)}  & \textbf{0.9116 (0.011)} & \textbf{0.9125 (0.012)} \\
      & & 100 & \textbf{0.0806 (0.007)} & \textbf{0.108 (0.007)} & \textbf{0.8956 (0.007)}  & \textbf{0.9028 (0.016)} \\
      & & 50 & \textbf{0.1031 (0.006)} & 0.137 (0.010)   & \textbf{0.8414 (0.029)} & \textbf{0.8537 (0.024)}\\
      \bottomrule
    \end{tabular}
  \label{tab8}
\end{table}

\begin{table*}[t]
	\centering
	\small
	\caption{Full results of the domain adaptation method under the semi-supervised setting. We compute the mean and standard deviation for the results of 20 trials. The baseline for Supervised Domain Adaptation (SDA) is indicated by $\ddagger$. The best results are \textbf{bolded}.}
	\label{tab11}
	\setlength{\tabcolsep}{4pt}
	\begin{tabular}{l|lccc|lccc}
		\toprule
		\multirow{2}{*}{Method} & \multicolumn{4}{c|}{$\mathcal{D}_G \rightarrow \mathcal{D}_M$} & \multicolumn{4}{c}{$\mathcal{D}_G \rightarrow \mathcal{D}_E$} \\
		  & Full & $20\%$ & $10\%$ & $5\%$ & Full &$20\%$ & $10\%$ & $5\%$ \\
		\midrule
		PnP-GA & 6.32$^{\circ}$(0.24) & 7.99$^{\circ}$(0.29) & 8.42$^{\circ}$(0.25) & 9.85$^{\circ}$(0.45) & 6.85$^{\circ}$(0.34) &8.54$^{\circ}$(0.29) & 8.95$^{\circ}$(0.96) & 11.78$^{\circ}$(0.70) \\
        PnP-GA+CLSS & -- & 6.59$^{\circ}$(0.19) & 8.16$^{\circ}$(0.62) & 11.3$^{\circ}$(1.72) & -- &7.20$^{\circ}$(0.13) & 8.43$^{\circ}$(0.67) & 12.47$^{\circ}$(0.94) \\
		\textbf{PnP-GA+DSCL} & -- & \textbf{6.16$^{\circ}$(0.12)} & \textbf{6.55$^{\circ}$(0.30)}& \textbf{6.62$^{\circ}$(0.13)} & -- &\textbf{6.76$^{\circ}$(0.31)} & \textbf{6.93$^{\circ}$(0.57)} & \textbf{7.58$^{\circ}$(0.15)}  \\
		\midrule
		UnReGA & 5.80$^{\circ}$ & -- & -- & -- & 5.42$^{\circ}$ & -- & -- & -- \\
		UnReGA$^{-}$ & 6.21$^{\circ}$(0.10) &6.60$^{\circ}$(0.17) & 8.31$^{\circ}$(0.22) & 9.55$^{\circ}$(0.33) & 6.40$^{\circ}$(0.07) & 8.33$^{\circ}$(0.13) & 10.11$^{\circ}$(0.19) & 14.69$^{\circ}$(0.34) \\
        UnReGA$^{-}$+CLSS & -- &6.54$^{\circ}$(0.07) & 7.28$^{\circ}$(0.13) & 10.6$^{\circ}$(0.19) & -- & 7.27$^{\circ}$(0.21) & 9.21$^{\circ}$(0.44) & 13.60$^{\circ}$(0.60) \\
		\textbf{UnReGA$^{-}$+DSCL} & -- & \textbf{6.18$^{\circ}$(0.06)} & \textbf{7.18$^{\circ}$(0.07)} & \textbf{9.22$^{\circ}$(0.21)} & -- & \textbf{6.91$^{\circ}$(0.08)} & \textbf{8.80$^{\circ}$(0.22)} & \textbf{11.52$^{\circ}$(0.68)} \\
		\midrule
		Baseline$\ddagger$ & 6.28$^{\circ}$(0.13) & 7.47$^{\circ}$(0.09) & 7.17$^{\circ}$(0.34) & 9.02$^{\circ}$(0.29) & 7.26$^{\circ}$(0.06) & 9.37$^{\circ}$(0.06) & 10.1$^{\circ}$(0.22) & 10.66$^{\circ}$(0.63) \\
        Baseline$\ddagger$+CLSS & -- & 7.23$^{\circ}$(0.10) & 7.09$^{\circ}$(0.10) & 8.75$^{\circ}$(0.13)& -- & 9.28$^{\circ}$(0.21)& 9.28$^{\circ}$(0.19)& 10.18$^{\circ}$(0.47)\\
		\textbf{Baseline$\ddagger$+DSCL} & -- & \textbf{6.63$^{\circ}$(0.13)} & \textbf{6.95$^{\circ}$(0.37)} & \textbf{7.22$^{\circ}$(0.11)} & -- & \textbf{8.15$^{\circ}$(0.09)} & \textbf{8.23$^{\circ}$(0.13)} & \textbf{8.64$^{\circ}$(0.39)}\\
		\bottomrule
	\end{tabular}
\end{table*}
\section{Results for DA methods with Standard Deviations}
In \cref{tab11}, we report the full results of domain adaptation method under semi-supervised setting.

\section{More Visualization}
\begin{figure*}
    \centering
    \includegraphics[width=\linewidth]{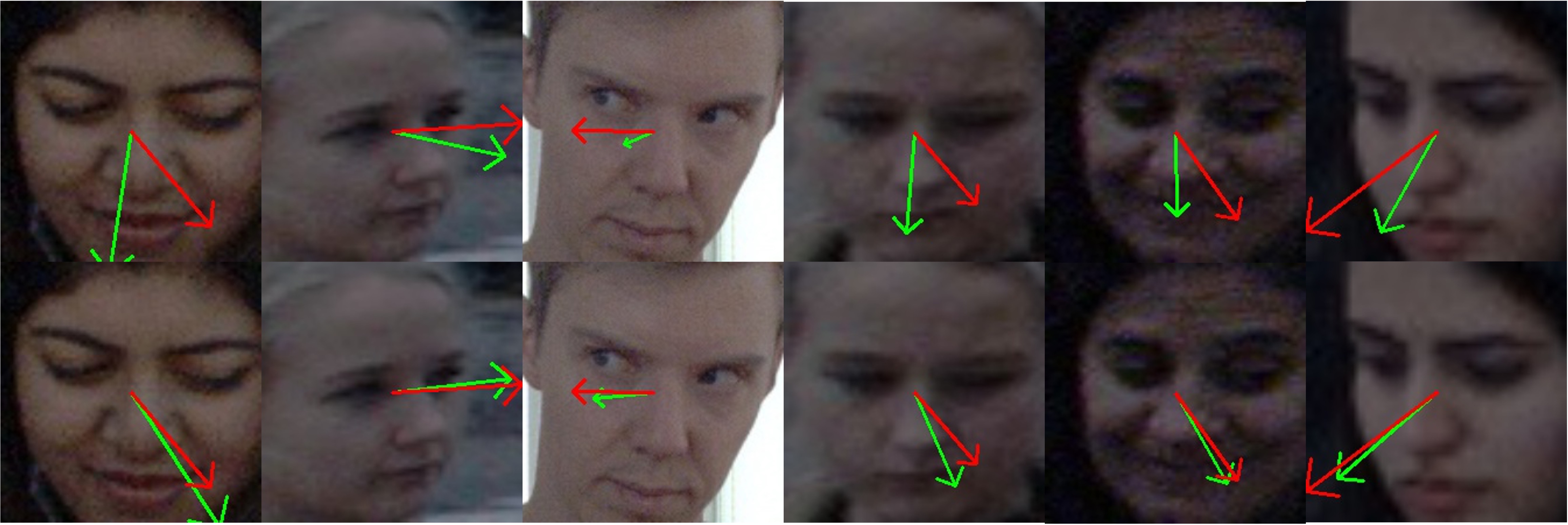}
    \caption{Visualized results for in-domain evaluations on Gaze360 testing set. The first row is from our baseline model, and the bottom row is from our DSCL method (i.e., Baseline + DSCL). The red and green arrows denote the ground truth and predicted gaze vectors, respectively.}
    \label{appen_fig2}
\end{figure*}

\begin{figure*}
    \centering
    \begin{minipage}{1.\linewidth}
        \includegraphics[width=1.0\linewidth]{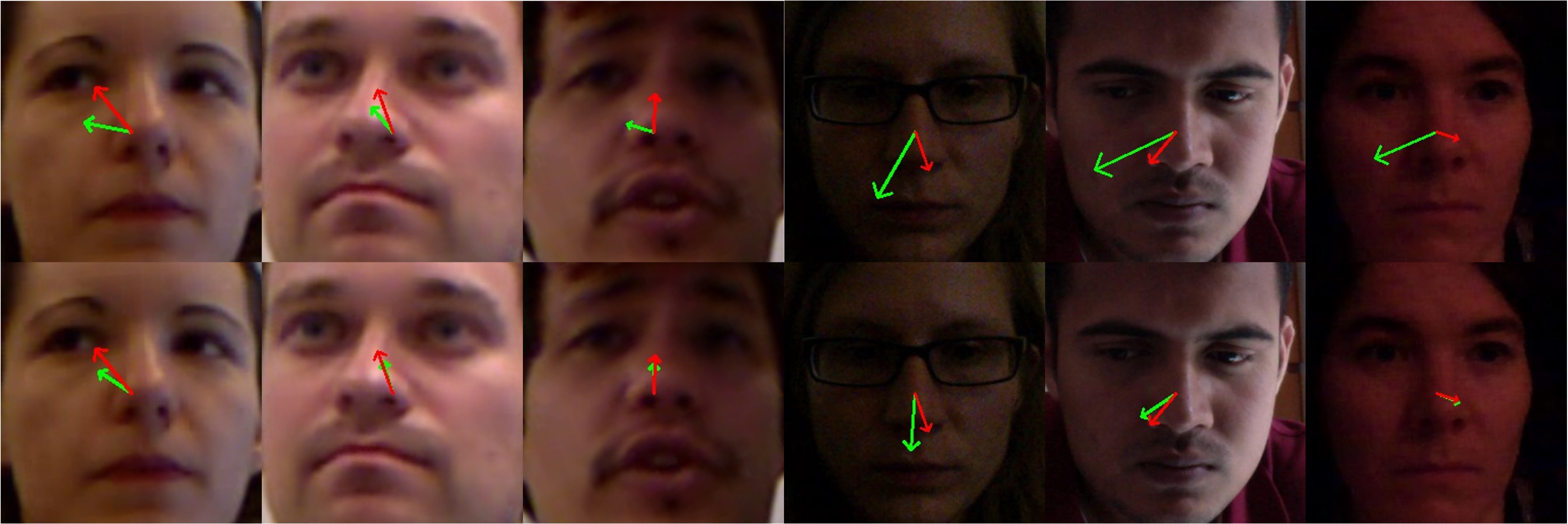}
        \centering{(a) PureGaze \textit{vs.} PureGaze + DSCL}
    \end{minipage}
    \begin{minipage}{1.\linewidth}
        \includegraphics[width=1.0\linewidth]{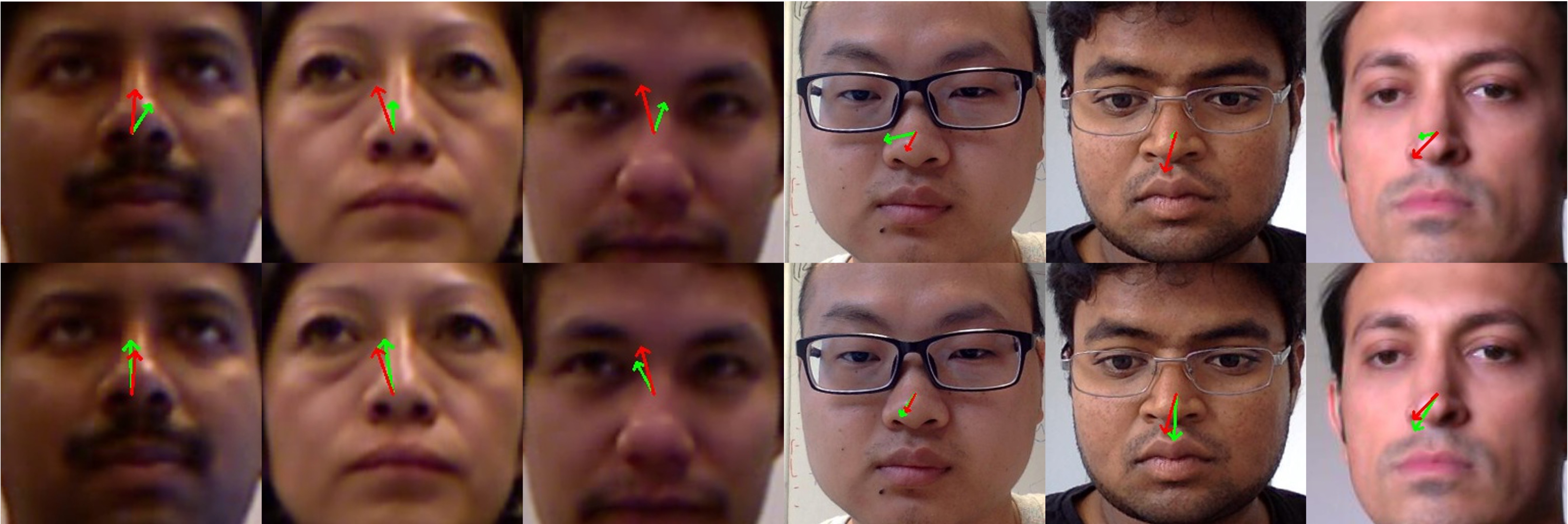}
        \centering{(b) Baseline \textit{vs.} Baseline + DSCL}
    \end{minipage}
    \caption{Visualized results for domain-generalization (DG) evaluations on MPIIGaze and EyeDiap datasets. The first row is from the original model, and the bottom row is from our DSCL method (i.e., PureGaze/Baseline + DSCL). The red and green arrows denote the ground truth and predicted gaze vectors, respectively.}
    \label{appen_fig3}
\end{figure*}

\begin{figure*}
    \centering
    \begin{minipage}{1.\linewidth}
        \includegraphics[width=1.0\linewidth]{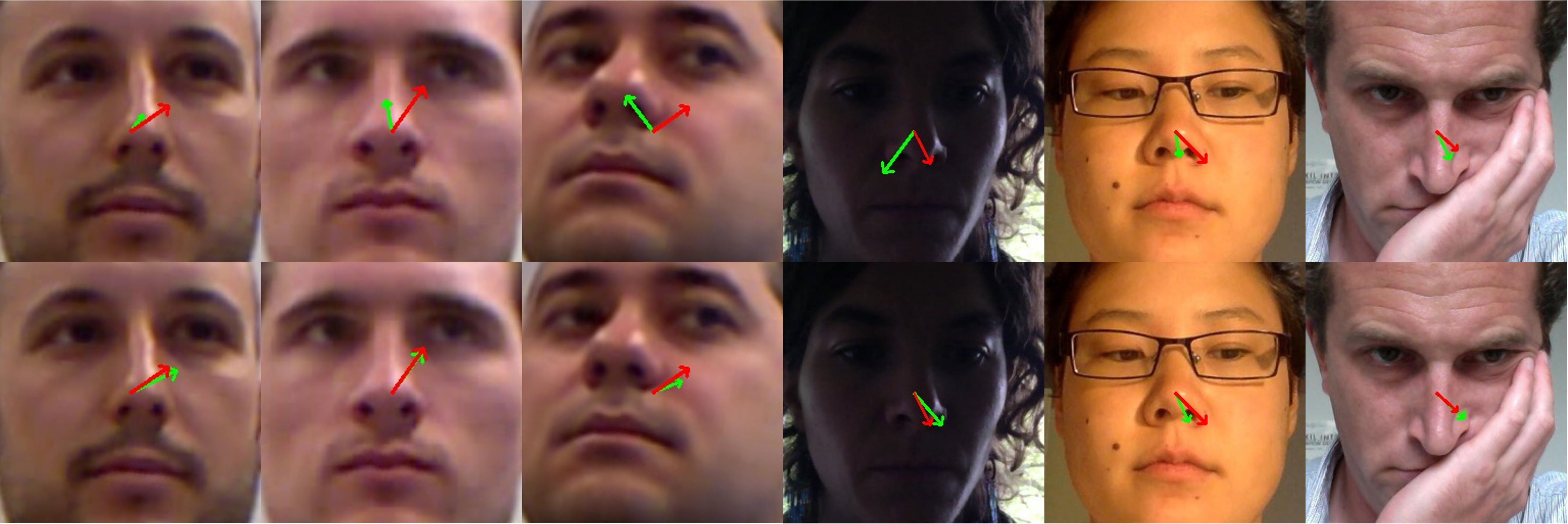}
        \centering{(a) UnReGa \textit{vs.} UnReGa + DSCL}
    \end{minipage}
    \begin{minipage}{1.\linewidth}
        \includegraphics[width=1.0\linewidth]{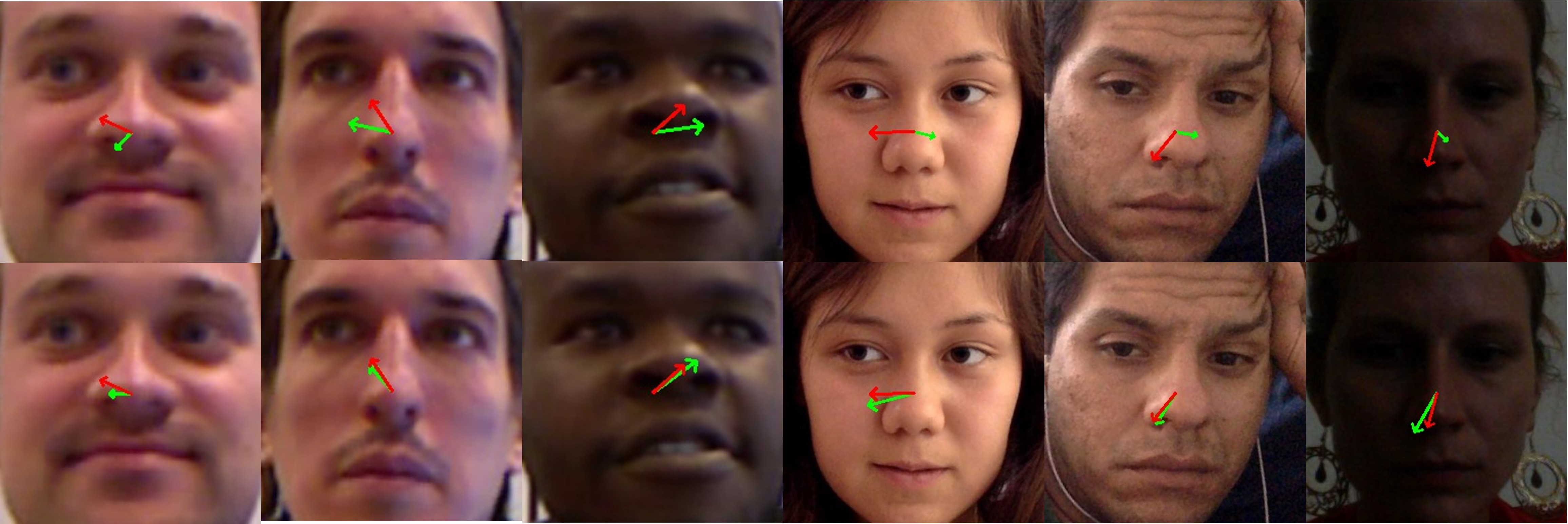}
        \centering{(b) PnP-GA \textit{vs.} PnP-GA + DSCL}
    \end{minipage}
    \caption{Visualized results for domain-adaptation (DA) evaluations on MPIIGaze and EyeDiap datasets under $10\%$ semi-supervised setting. The first row is from the original model, and the bottom row is from our DSCL method (i.e., PnP-GA/UnReGa + DSCL). The red and green arrows denote the ground truth and predicted gaze vectors, respectively.}
    \label{appen_fig4}
\end{figure*}

\begin{figure*}
    \centering
    \begin{minipage}{1.\linewidth}
        \includegraphics[width=1.0\linewidth]{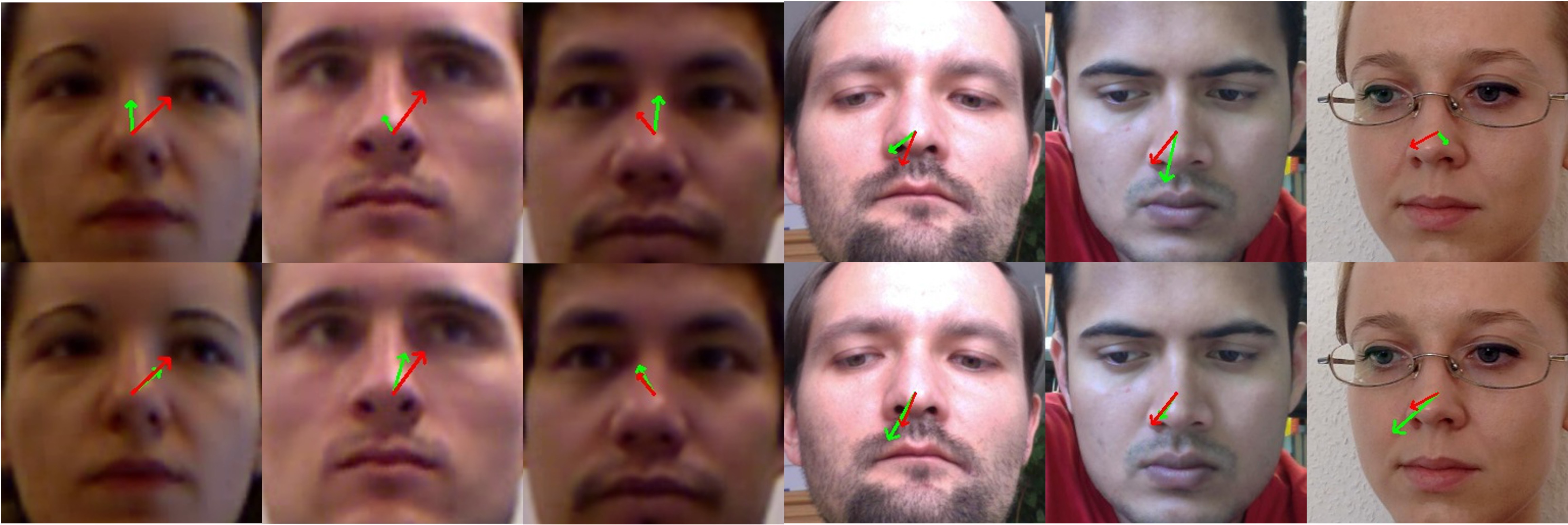}
    \end{minipage}
    \caption{Visualized results for supervised domain-adaptation (SDA) evaluations on MPIIGaze and EyeDiap datasets under $10\%$ semi-supervised setting. The first row is from our baseline model, and the bottom row is from our DSCL method (i.e., Baseline + DSCL). The red and green arrows denote the ground truth and predicted gaze vectors, respectively.}
    \label{appen_fig5}
\end{figure*}

We first visualize the results for in-domain evaluation, where we only exploit $10\%$ labeled samples, and the rest are used as unlabeled sample on the Gaze360, the results are shown in \cref{appen_fig2}. It is obvious that our baseline model achieves coarse predictions due to lack of sufficient labeled training data. Instead, integrating our DSCL framework into the baseline, it brings significant improvements.

Furthermore, we visualize the results for cross-domain evaluations. We first evaluate our method for the domain-generalization gaze estimation, where the model is trained on source domain only, and tested on target domain. \cref{appen_fig3} gives the results under the $10\%$ semi-supervised setting. Similarly, through integrating our DSCL into the existing framework, which significantly improves the generalization ability for cross-domain evaluation, even with a few labeled samples.

Moreover, we give the more visualized results for domain-adaptation (DA) evaluation, where we select some representative DA-based methods, e.g., UnReGa \cite{cai2023source} and PnP-GA \cite{liu2021generalizing}, and then inject our DSCL into them, results are shown in \cref{appen_fig4}. Our DSCL also brings obvious improvements for domain-adaptation gaze estimation, which supports our DSCL can learn more robust gaze representation from labeled and unlabeled samples.

\cref{appen_fig5} gives the visualized results for supervised domain-adaptation (SDA) evaluation, where we exploit the labeled target-domain samples during the adaptation stage. Similarly, our DSCL also demonstrates strong performance on cross-domain evaluation with a few labeled samples, which is more applicable for real-world scenarios.
\end{document}